\documentclass[twoside,11pt]{article}

\usepackage{blindtext}

 \usepackage[abbrvbib, preprint]{jmlr2e}

\usepackage{amsmath}

\usepackage{caption}
\usepackage{subcaption}
\usepackage{graphicx}
\usepackage{bbm}
\usepackage{dsfont}
\usepackage{xcolor}
\usepackage{bbm}
\usepackage{enumitem}
\usepackage[ruled, vlined]{algorithm2e}

\usepackage{algpseudocode}%

\DeclareMathOperator*{\argmin}{arg\,min}

\newtheorem{assumption}{Assumption}

\newcommand{\eps}{\varepsilon}

\usepackage{lastpage}
\jmlrheading{25}{2025}{1-\pageref{LastPage}}{1/19}{}{}{Carlos Misael Madrid Padilla, Shitao Fan and Lizhen Lin}

\ShortHeadings{Robust and Scalable Variational Bayes}{Madrid Padilla, Fan and Lin}
\firstpageno{1}

\begin{document}

\title{Robust and Scalable Variational Bayes}

\author{\name Carlos Misael Madrid Padilla \email carlos.madrid@cimat.mx \\
       \addr Department of Statistics and Data Science\\
       Washington University in St. Louis\\
       St. Louis, MO 63130, USA
       \AND
       \name Shitao Fan \email sfan211@umd.edu \\
       \addr Department of Mathematics, \\
       University of Maryland\\
       College Park, MD 20742-4015, USA
       \AND
       \name Lizhen Lin \email lizhen01@umd.edu \\
       \addr Department of Mathematics, \\
       University of Maryland\\
       College Park, MD 20742-4015, USA
       }
\editor{}

\maketitle

\begin{abstract}
We propose a robust and scalable framework for variational Bayes (VB) that effectively handles outliers and  contamination of arbitrary nature in large datasets. Our approach divides the dataset into disjoint subsets, computes the posterior for each subset, and applies VB approximation independently to these posteriors. The resulting  variational posteriors with respect to the subsets are then aggregated using the geometric median of probability measures, computed with respect to the Wasserstein distance. This novel aggregation method yields the \emph{Variational Median Posterior} (VM-Posterior) distribution.  
We rigorously demonstrate that the VM-Posterior preserves contraction properties akin to those of the true posterior,  while accounting for approximation errors or the variational gap inherent in VB methods. We  also provide provable robustness guarantee of the VM-Posterior. Furthermore, we establish a variational Bernstein–von Mises theorem for both multivariate Gaussian distributions with general covariance structures and the mean-field variational family. To facilitate practical implementation, we adapt existing algorithms for computing the VM-Posterior and evaluate its performance through extensive numerical experiments. The results highlight its robustness and scalability, making it a reliable tool for Bayesian inference in the presence of complex, contaminated datasets.
\end{abstract}


\section{Introduction}
Bayesian inference is a foundational paradigm in  statistics and machine learning, providing a rigorous framework for probabilistic modeling of unknown parameters and making predictions under uncertainty. Despite its theoretical appeal,  practical application of Bayesian methods can face challenges due to the reliance on Markov Chain Monte Carlo (MCMC) methods for sampling the posterior distributions for inference, which is often computationally intractable, especially when dealing with large datasets and complicated models. 


To address these limitations, variational Bayes (VB) has emerged as a computationally efficient alternative. VB approximates the posterior distribution by optimizing over a simpler family of distributions, effectively transforming the MCMC sampling problem into a tractable optimization  problem. This shift significantly reduces computational costs, making VB particularly suitable for large-scale data applications.
Theoretical guarantees for VB methods have also advanced recently (see \cite{wang2019frequentist}, \cite{zhang2020convergence}, \cite{yang2020alpha} and \cite{ohn2024adaptive} ). These developments have further cemented VB as a practical and scalable solution that's backed up by theory. VB's computational efficiency and adaptability have driven its adoption across diverse fields such as natural language processing [e.g.  \cite{bowman2015generating}, \cite{li2021prefix}], Bayesian deep learning [e.g \cite{nazaret2022variational},  \cite{sen2024bayesian}], genomics and bioinformatics [e.g. \cite{raj2014faststructure}, \cite{komodromos2022variational}], and healthcare analytics [e.g. \cite{zabad2023fast}, \cite{koh2024variational}].

Despite the computational efficiency of VB, handling outliers and contamination remains a significant challenge in the VB framework. An outlier—following \cite{box1968bayesian}—is an observation suspected of being partially or entirely irrelevant because it does not conform to the assumed stochastic model. 
Outliers and contaminations disrupt statistical inference, often introducing bias or compromising the accuracy of results. Traditional approaches to handling outliers, particularly in point estimation (e.g., \cite{Huber2011} and \cite{law1986robust}), have achieved considerable success. However, Bayesian methods, including those within the VB framework, frequently rely on strong distributional assumptions or preprocessing techniques to mitigate the effects of outliers. For instance,  \cite{giordano2018covariances}, while not explicitly focuses on outliers,    highlights how VB methods can misestimate variances and covariances when outliers are present, emphasizing the critical role of data-cleaning decisions.
Similarly, \cite{futami2018variational} utilizes a robust VB method where outliers are handled using a heavy-tailed distribution to directly account for anomalies in the data, improving the robustness of the inference.  
Additionally, \cite{JIN2012423} employs a heavy-tailed t-distribution in a hierarchical variational framework to solve inverse problems in the presence of outliers without requiring preprocessing steps, demonstrating robustness and convergence. Moreover, \cite{5582315} introduces a Bayesian autoregression model that uses Student-t distributed noise to manage outliers in time series data, enhancing performance without preprocessing. Some VB methods, such as \cite{9874893}, specifically handle outliers in applications like forward-looking imaging by using a Student-t distribution for non-Gaussian noise, although such approaches often lack theoretical guarantees. The reliance on such assumptions and preprocessing techniques can limit the flexibility of these methods, especially in large-scale applications. Additionally, not all robust VB techniques provide theoretical guarantees, underscoring the need for further advancements in this area.

In addition to the challenges faced by variational Bayes (VB) approaches, handling outliers and data contamination also poses significant difficulties in the broader context of standard Bayesian analysis. Many Bayesian methods address these issues by assuming specific noise distributions or relying on preprocessing steps to filter out anomalies. These strategies, while effective in certain specific settings, often compromise scalability and robustness, particularly when the true posterior distribution is intractable.

An notable exception to these limitations is the M-Posterior method proposed by \cite{minsker2017robust}. This approach partitions the data into non-overlapping subgroups, computes the posterior distribution for each subgroup independently, and combines the results by taking the median in the space of probability measures, using the Wasserstein distance.  The method provides strong theoretical guarantees, ensuring robustness in the presence of contamination.

While the M-posterior method in \cite{minsker2017robust} is provably robust,
the need for MCMC sampling for each subset posterior can become computationally expensive, particularly for large datasets or high-dimensional problems where sampling from the true posterior is inherently challenging.
We address these gaps by proposing a novel variational Bayes (VB) approach that eliminates the need for MCMC sampling for each subset posterior distribution while maintaining robustness and computational efficiency. Specifically, we introduce the Variational Median Posterior (VM-Posterior) method, which integrates the M-Posterior framework by \cite{minsker2017robust} with variational inference.
The key innovation lies in combining the computational efficiency of VB with the robustness of the median aggregation step using the Wasserstein distance. The VM-Posterior method operates as follows:
\begin{itemize}
    \item 
Data Partitioning and VB-approximation: The dataset is divided into disjoint subsets, and the posterior distribution with respect to each subset is approximated by  a variational posterior.
\item Robust Aggregation: The subset variational posteriors are combined using the geometric median in the space of probability measures, leveraging the Wasserstein distance to ensure robustness against outliers.
\end{itemize}
We will show that this approach effectively handles multiple outliers, regardless of their magnitude or nature, making it particularly suitable for large, real-world datasets often contaminated with anomalies.  More specifically, our contributions can be summarized as follows: 

\begin{itemize}


\item \textbf{Contraction and robustness 
Properties}: We  develop a novel contraction analysis tailored to the Variational Median Posterior (VM-Posterior).
Our approach involves first deriving the contraction rate of the original posterior to $\theta_0$ with rate $\epsilon_l$. Subsequently, we bound the additional error introduced by the variational approximation, i.e.,  \textit{the variational approximation gap}, by an upper bound of order $l\epsilon_l^2.$ Like the M-Posterior in \cite{minsker2017robust}, the VM-Posterior achieves this contraction rate even in the presence of outliers, demonstrating its robustness.



This approach is valuable as it establishes the core contraction rate $\epsilon_l$ and quantifies the impact of the variational approximation through the variational gap, reflecting the trade-off between computational efficiency and fidelity to the true posterior. Bounding this gap at order $l\epsilon_l^2$ ensures the VM-Posterior maintains robustness while remaining computationally scalable. By carefully selecting the variational family, this method minimizes the gap, balancing accuracy and efficiency.

Our work extends established frameworks, such as those by \cite{zhang2020convergence} and \cite{ohn2024adaptive}, to the VM-Posterior, showing for the first time that it preserves contraction properties while being robust to multiple outliers. 


\item \textbf{BVM (Berstein von-Mise theorem) for the VM-Posterior}:  

We analyze the asymptotic properties of the VM-Posterior for two variational families: mean-field class and Gaussian family with general covariance. In both cases, the VM-Posterior asymptotically follows a normal distribution centered at a robust estimator  $\theta^*$, which approximates the true parameter $\theta_0$ with accuracy governed by the variational gap.

The key distinction lies in the covariance structure of the limiting distribution: it is diagonal in the mean-field case (independent uncertainty) and full in the Gaussian case (parameter dependencies). In both cases, the covariance matrix is of the form $\frac{I^{-1}(\theta_0)}{n}$, where $I(\theta_0)$ is the Fisher information matrix.

These results parallel the Bernstein–von Mises theorem, showing that the VM-Posterior retains the asymptotic normality of the true posterior, ensuring valid uncertainty quantification. This provides a solid theoretical foundation for using VM-Posterior Bayesian inference across both independent and dependent parameter models.

\item \textbf{Practical Implications}:
We conduct comprehensive numerical experiments on both simulated and real-world datasets to validate our theoretical findings. The experiments focus on comparing the performance of our proposed VM-Posterior against the M-Posterior  in terms of computational efficiency and robustness to outliers.


Figures \ref{fig:intro_gauss} and \ref{fig:intro_lda} illustrate the computational advantages of the VM-Posterior, particularly in high-dimensional models and datasets with large outliers. The VM-Posterior demonstrates consistently lower computational times compared to the M-Posterior, which becomes significantly slower as outlier magnitude increases. These results highlight the VM-Posterior’s practicality for large-scale Bayesian inference tasks where both robustness and computational efficiency are crucial.
\begin{figure}[htbp!]
     \centering     \begin{subfigure}[b]{0.45\textwidth}
         \centering
         \includegraphics[width=\textwidth]{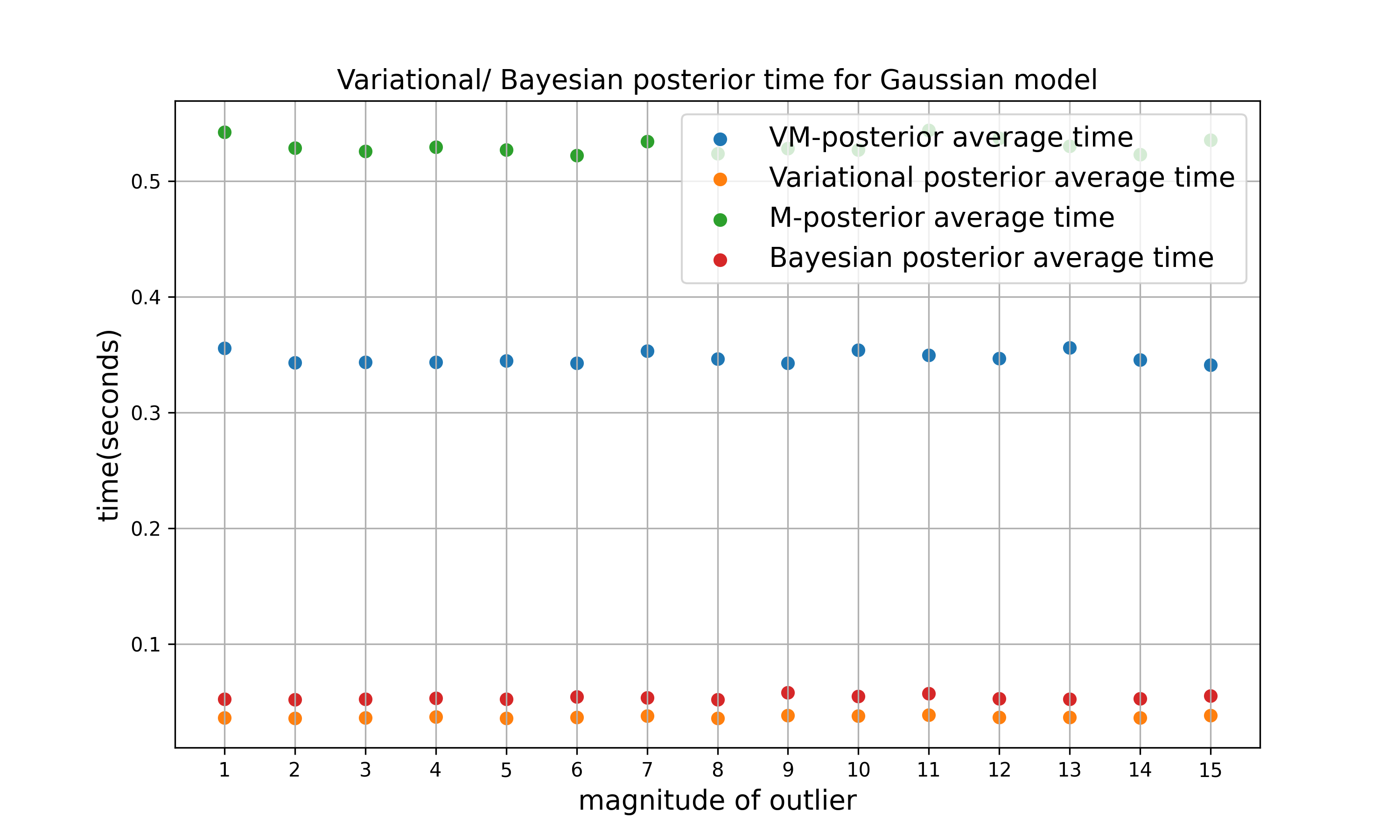}
         \caption{Computational Time Comparison for Gaussian model}
         \label{fig:intro_gauss}
     \end{subfigure}
     \hfill
     \begin{subfigure}[b]{0.45\textwidth}
         \centering
         \includegraphics[width=\textwidth]{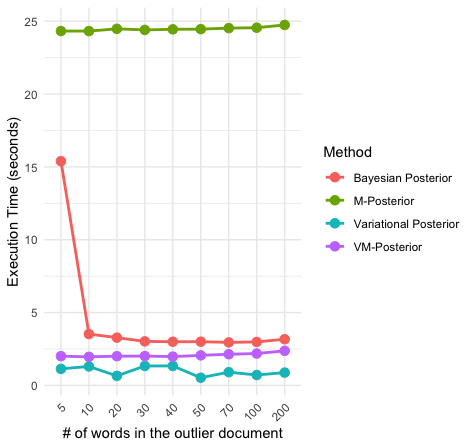}
         \caption{Computational time comparison for LDA model}
         \label{fig:intro_lda}
     \end{subfigure}
    \caption{Computational efficiency for VM-Posterior}
    \end{figure}

Our further results  in Section \ref{simu-data} also confirm the robustness of the VM-Posterior in terms of posterior coverage.  The coverage closely matches expected levels, across various magnitudes of outliers and different significance levels.
This robustness is as good as that of the M-Posterior, but the VM-Posterior stands out due to its significantly reduced computational cost. Moreover, variational Bayesian approach are often favored in many machine learning algorithm, where classic Bayesian methods may fail or computationally untractable.


\end{itemize}

\subsection{Outline}

The paper is organized as follows: Section \ref{sec-notation} defines the notation used, while Section \ref{VM-approach-def} introduces the methodology. Section \ref{SectionPartiotioning} explains the model setup and data partition, followed by Section \ref{adjutlikelihood}, which discusses likelihood power adjustment for partitioned posteriors. The variational approach, variational family, and contraction rates are covered in Section \ref{Variationa-approx}. Section \ref{medianVP} introduces the robust aggregation with respect to the Wasserstein distance leading to the VM-Posterior method.
The theoretical foundations of our approach are established in Sections \ref{robustness},
with Section \ref{sec-Asympnorm} exploring the Bernstein-von Mises theorem. Section \ref{AlgoSection} outlines the algorithms, including Section \ref{sec-AlgoVI} for the variational approach, Section \ref{sec-algowass} for the geometric median, and Section \ref{sec-mgauss}, \ref{sec-gmm}, and \ref{sec-kerneldist} for Gaussian models, Gaussian mixtures, and discrete distributions, respectively. 
Section \ref{simu-data} presents simulation studies, including evaluations on multivariate Gaussian models, Gaussian mixtures, and latent Dirichlet allocation, as well as real data analysis in Section \ref{sec-realdata}. The paper concludes in Section \ref{sec-conclusion}, summarizing findings and future research directions.

\subsection{Notation}
\label{sec-notation}
In what follows, $\|\cdot\|_2$ denotes the standard Euclidean distance in $\mathbb{R}^p$ and $\langle\cdot, \cdot\rangle_{\mathbb{R}^p}$ the associated dot product.

Given a totally bounded metric space $(\mathbb{Y}, d)$, the packing number $M(\varepsilon, \mathbb{Y}, d)$ is the maximal number $N$ such that there exist $N$ disjoint $d$-balls $B_1, \ldots, B_N$ of radius $\varepsilon$ contained in $\mathbb{Y}$, i.e., $\bigcup_{j=1}^N B_j \subseteq \mathbb{Y}$.

Let $\left\{p_\theta, \theta \in \Theta\right\}$ be a family of probability density functions on $\mathbb{R}^p$. Let $l, u: \mathbb{R}^p \mapsto \mathbb{R}_{+}$ be two functions such that $l(x) \leq u(x)$ for every $x \in \mathbb{R}^p$ and $d^2(l, u):=\int_{\mathbb{R}^p}(\sqrt{u}-\sqrt{l})^2(x) d x<\infty$. A bracket $[l, u]$ consists of all functions $g: \mathbb{R}^p \mapsto \mathbb{R}$ such that $l(x) \leq g(x) \leq u(x)$ for all $x \in \mathbb{R}^p$. For $A \subseteq \Theta$, the bracketing number $N_{[]}(\varepsilon, A, d)$ is defined as the smallest number $N$ such that there exist $N$ brackets $\left[l_i, u_i\right], i=1, \ldots, N$ satisfying $\left\{p_\theta, \theta \in A\right\} \subseteq \bigcup_{i=1}^N\left[l_i, u_i\right]$ and $d\left(l_i, u_i\right) \leq \varepsilon$ for all $1 \leq i \leq N$.

For $y \in \mathbb{Y}, \delta_y$ denotes the Dirac measure concentrated at $y$. In other words, for any Borel-measurable $B, \delta_y(B)=I\{y \in B\}$, where $I\{\cdot\}$ is the indicator function.

We will say that $k: \mathbb{Y} \times \mathbb{Y} \mapsto \mathbb{R}$ is a kernel if it is a symmetric, positive definite function. Assume that $\left(\mathbb{H},\langle\cdot, \cdot\rangle_{\mathbb{H}}\right)$ is a reproducing kernel Hilbert space (RKHS) of functions $f: \mathbb{Y} \mapsto \mathbb{R}.$ Then $k$ is a reproducing kernel for $\mathbb{H}$ if for any $f \in \mathbb{H}$ and $y \in \mathbb{Y},\langle f, k(\cdot, y)\rangle_{\mathbb{H}}=f(y)$ (see Aronszajn, 1950 for details).

The Kullback-Leibler (KL) divergence between two probability density functions $p(x)$ and $q(x)$ is defined as:
\[
\text{KL}(p \| q) := \int_{\mathbb{R}^p} p(x) \log\left(\frac{p(x)}{q(x)}\right) dx,
\]
provided the support of $p(x)$ is contained within the support of $q(x)$. The KL divergence measures the difference between two probability distributions and is frequently used in variational inference to quantify how well the approximate distribution $q(x)$ matches the true posterior $p(x)$.

We use the notation $P_0^l$ to denote both the expected value with respect to the random variables $(X_1, \ldots, X_l)$ that are i.i.d. under $\theta_0$, and the probability measure associated with these random variables. In particular, $P_0^l(g) = \mathbb{E}_{P_0^l}[g(X_1, \ldots, X_l)]$ for any measurable function $g$, and for any measurable set $B \subseteq \mathbb{R}^{pl}$, $P_0^l(B)$ represents the probability of $B$ under this measure.


\section{The VM-Posterior}
\label{VM-approach-def}

In this section, we present the VM-Posterior approach, a robust method for combining variational posteriors. This approach leverages either the \textit{Wasserstein geometric median} or the \textit{Wasserstein metric median} to construct a final posterior measure. These methods ensure robustness against outliers while maintaining computational efficiency, making the VM-Posterior a practical and reliable solution for large-scale Bayesian inference.

\subsection{Model Setup and Data Partitioning}
\label{SectionPartiotioning}

Let $ \{P_\theta, \theta \in \Theta\} $ be a family of probability distributions over $ \mathbb{R}^D $, where $ \Theta  \subset \mathbb{R}^D$ is the parameter space. For any $ \theta \in \Theta $, $ P_\theta $ is absolutely continuous with respect to the Lebesgue measure $ dx $ on $ \mathbb{R}^D $, with $ dP_\theta(\cdot) = p_\theta(\cdot) dx $. The space $\Theta$ is equipped with the \textit{Hellinger metric}, defined as \begin{equation}
\label{therho}
    \rho(\theta_1, \theta_2) := h(P_{\theta_1}, P_{\theta_2}), 
\end{equation}
where $h(\cdot, \cdot)$ denotes the Hellinger distance between probability distributions. We assume that $ (\Theta, \rho) $ is a separable metric space. 

In Bayesian inference, a prior distribution $ \Pi $ over $ \Theta $ is specified, where $\Theta$ is equipped with the Borel $\sigma$-algebra induced by $\rho$. The prior $ \Pi $ encodes initial beliefs about the unknown  true parameter $ \theta_0 $ before any data is observed, and these beliefs are subsequently updated with the data to form a posterior distribution. Let $ X_1, \ldots, X_n $ be i.i.d. $ \mathbb{R}^D $-valued random vectors defined on a probability space $ (\Omega, \mathcal{B}, P) $, with unknown true distribution $ P_0 = P_{\theta_0} $ for some $ \theta_0 \in \Theta $. Given observed data $ \mathbf{X}_n = \{X_1, \ldots, X_n\} $, the posterior distribution on $\Theta$ is defined as
\[
\Pi_n(B \mid \mathbf{X}_n) := \frac{\int_B \prod_{i=1}^n p_\theta(X_i) \, d\Pi(\theta)}{\int_{\Theta} \prod_{i=1}^n p_\theta(X_i) \, d\Pi(\theta)}
\]
for all Borel measurable sets $ B \subseteq \Theta $. Under general conditions, the posterior distribution $\Pi_n$ is known to contract around the true parameter $\theta_0$ as $n \rightarrow \infty$  with \emph{contraction rates $\epsilon_n$} (see \cite{ghosal2000convergence}), if for a suitable sequence $\epsilon_n \rightarrow 0$, the following holds:
\[
\Pi_n\left(\theta \in \Theta: \rho\left(\theta, \theta_0\right) \geq M\varepsilon_n \mid \mathbf{X}_n\right) \rightarrow 0,
\]
almost surely or in probability as $n$ grows for an arbitrary $M\rightarrow \infty$, indicating that the posterior becomes increasingly concentrated around $\theta_0$. However, the contraction property may be compromised in practical scenarios where the dataset $\mathbf{X}_n$  contains outliers of arbitrary nature and magnitude. For instance, even a single outlier may cause the posterior to deviate substantially from $\theta_0$, disrupting the concentration behavior outlined above. Moreover, in scenarios where $ \mathbf{X}_n $ is large, computing the full posterior $\Pi_n$ is often computationally prohibitive due to memory or processing constraints, raising scalability issues. 

To address both scalability and robustness, we propose to partition the sample into $m$ disjoint groups $ G_1, \ldots, G_m $, with each group containing at least $ \lfloor n/m \rfloor $ observations. This is,
\[
\mathbf{X}_n=\bigcup_{j=1}^m G_j, \quad G_i \cap G_l=\emptyset \text{ for } i \neq j, \quad \left|G_j\right| \geq\lfloor n / m\rfloor, \quad j=1, \ldots, m, \quad 1\le m\le \frac{n}{2}.
\]
Our idea is to obtain a subset posterior with respect to the above subset data (after likelihood power adjustment described in the next subsection), provide an variational approximation to teach, which are then aggregated to obtain the final VM-posterior. This disjoint grouping of the dataset brings several significant advantages. First, it substantially improves computational efficiency by allowing posterior approximations to be computed in parallel across subsets, effectively addressing challenges posed by centralized processing (see, e.g., \cite{wang2014scalable}, \cite{wang2015parallelizing}). Additionally, processing subsets independently reduces memory demands, which is crucial for large datasets requiring inference on manageable data chunks. This parallel strategy is foundational in modern large-scale Bayesian inference frameworks, enabling scalable computation, see for instance  \cite{srivastava2018scalable} and \cite{peruzzi2022spatial}.


\subsection{Likelihood Power Adjustment}
\label{adjutlikelihood}

Partitioning offers important benefits for handling outliers and scaling large data. However, it can also inflate uncertainty around $\theta_0$. This issue becomes more pronounced as the number of groups, $m$, increases. Specifically, consider the situation where $\theta \in \mathbb{R}$, and the Bernstein-von Mises theorem, see \cite{van2000asymptotic}, holds. Under these conditions, each subset-based posterior, $\Pi_{\vert G_j \vert}(\cdot \mid G_j)$, defined as 
\[
\Pi_{\left|G_j\right|}\left(B \mid G_j\right):=\frac{\int_B \prod_{i \in G_j} p_\theta\left(X_i\right) \, d \Pi(\theta)}{\int_{\Theta} \prod_{i \in G_j} p_\theta\left(X_i\right) \, d \Pi(\theta)},
\]
will be approximately normal, with an asymptotic covariance of $\frac{m}{n} I^{-1}(\theta_0)$, where $I(\theta_0)$ denotes the Fisher information. This covariance is larger than that of the posterior distribution based on the entire sample, which would asymptotically be $\frac{1}{n} I^{-1}(\theta_0)$. Consequently, the subset-based posteriors $\Pi_{\left|G_j\right|}$ may produce an artificially inflated uncertainty estimate.

To mitigate this issue, following \cite{minsker2014scalable}, we apply a likelihood power adjustment to each subset-based posterior. Concretely, we modify each $\Pi_{\left|G_j\right|}$ by raising its likelihood to the power $m$. This adjusted posterior, named \textit{stochastic approximation} and denoted as $\Pi_{n,m}^{\vert G_j \vert}(B) = \Pi_{n,m}(B \mid G_j)$, is given by

\[
\Pi_{n,m}^{\vert G_j \vert}(B) := \frac{\int_B \left( \prod_{i \in G_j} p_\theta(X_i) \right)^m d\Pi(\theta)}{\int_{\Theta} \left( \prod_{i \in G_j} p_\theta(X_i) \right)^m d\Pi(\theta)}.
\]
The raised likelihood $\left(\prod_{i \in G_j} p_\theta(X_i)\right)^m$ can be interpreted as replicating each observation in $G_j$ a total of $m$ times.  Applying the adjustment reduces inflated variance in subset-based posteriors, making each more representative of the full dataset. Each adjusted posterior $\Pi_{n,m}^{\vert G_j \vert}(\cdot \mid G_j)$ thus approximates the full posterior more closely than its unadjusted counterpart, as it behaves as if each data point in $G_j$ is observed multiple times, creating a more accurate reflection of the information from the entire sample.

It turns out that the likelihood adjustment is essential for achieving realistic uncertainty across subsets. It aligns the subset-based posteriors with the full dataset’s information content. By this calibration, we achieve a realistic assessment of parameter uncertainty. Notably, as highlighted by \cite{srivastava2018scalable}, an advantage of the stochastic approximation approach is that it allows for the use of off-the-shelf sampling algorithms without additional computational load from actual data replication. In practice, the likelihood adjustment is often implemented by simply raising the likelihood to a power in full-data samplers. 
Combined with robust aggregation techniques, see Section \ref{medianVP}, the adjustment enhances posterior predictive credible regions and posterior coverage. 
The process yields regions that better reflect the true underlying parameter distributions as supported by our numerical results, see Section \ref{AlgoSection}.

\subsection{Variational Posterior Computation and Contraction Rates}
\label{Variationa-approx}

A key component in our approach is the use of variational inference to efficiently approximate \(\Pi_{n,m}^{|G_j|}(\cdot \mid G_j)\), $j=1,\ldots,m$,   the subset-based posterior distributions after the likelihood power. This approximation is achieved by minimizing the Kullback-Leibler (KL) divergence between a simpler, variational candidate posterior \(Q\) and the target subset-based adjusted posterior \(\Pi_{n,m}(\cdot \mid G_j)\),
\[
\widehat{Q}_{n,m}^{(j)} = \arg\min_{Q \in \mathcal{Q}} \textsc{KL}(Q, \Pi_{n,m}(\cdot \mid G_j)),
\]
where \(\mathcal{Q}\) denotes the variational family of candidate distributions. In the following, we refer to $\widehat{Q}_{n,m}^{(j)}$ as \textit{subset-based variational posterior}. The choice of \(\mathcal{Q}\)
 plays a critical role, as it determines the trade-off between computational efficiency and accuracy in the approximation. Among commonly used variational families, the mean-field family is particularly popular. This family assumes a fully factorized structure, allowing for computational efficiency. However, it may overlook dependencies between parameters. The mean-field family is defined as
\[
\mathcal{Q}_{\text{MF}} = \left\{ \prod_{j=1}^d q_j(\theta_j) \right\}.
\]
Recent research has extended variational approaches beyond the mean-field by adopting more flexible families. These families are designed to capture complex posterior structures, accommodating dependencies that simpler models may overlook. Examples include Gaussian distributions with general covariance structures and mixtures of such distributions, both of which offer greater flexibility for modeling dependencies and multi-modal characteristics. The Gaussian family with general covariance structures, denoted as $\mathcal{Q}_{\mathrm{GG}}$, is defined by
\[
\mathcal{Q}_{\text{GG}} = \left\{ N(\mu, \Sigma): \mu \in \mathbb{R}^D, \Sigma \in \mathbb{R}^{D \times D} \text{ is positive definite} \right\}.
\]
The use of this family is supported by the Bernstein–von Mises theorem, which shows that posterior distributions converge to a Gaussian form in large-sample settings (see, for instance, \cite{ray2022variational}). This theoretical basis underpins the effectiveness of the Gaussian family in approximating posteriors when sample sizes are large.
However, for more complex finite-sample settings, Gaussian mixtures with general covariance matrices, such as
\[
\mathcal{Q}_{\text{GGM}} = \left\{ \sum_{i=1}^s \pi_i N(\mu_i, \Sigma_i): \pi_i \geq 0, \sum_{i=1}^s \pi_i = 1, \mu_i \in \mathbb{R}^D, \Sigma_i \in \mathbb{R}^{D \times D} \right\}.
\]
are particularly beneficial. These families effectively capture multi-modal or skewed posteriors, as emphasized by \cite{zobay2014variational}, who explore the advantages of using mixtures of univariate Gaussians. Additionally, \cite{lin2022multisource} highlight the expressive capabilities of multivariate Gaussian mixtures, particularly with general covariance matrices, in handling intricate posterior structures.

 In practice, selecting an appropriate variational family $\mathcal{Q}$ involves balancing computational efficiency
 with approximation quality. 
 While the mean-field family $\mathcal{Q}_{\text {MF }}$ offers computational efficiency for large-scale datasets, the Gaussian family $\mathcal{Q}_{\text {GG }}$ provides a middle ground by capturing linear dependencies among parameters. Meanwhile, more complex families, such as Gaussian mixtures $\mathcal{Q}_{\mathrm{GGM}}$, are well-suited for accommodating multi-modal or skewed posteriors, making them advantageous for capturing non-Gaussian characteristics and nuanced dependencies.

  Formally, we introduce the following condition on the prior and variational family, which is important in establishing  the posterior contraction rates of the variational posterior distribution.
  This condition distinguishes itself from existing assumptions in the variational Bayes literature due to the power adjustment applied to the likelihood function, as described in Section \ref{adjutlikelihood}. 
  It reduces to a term that  captures the KL divergence between  the prior and  a variational element $Q$ and a term that accounts  for the data-generating process. 
\begin{assumption}
\label{Prior-cond}
(Prior and variational family) Let \(l > 0\). Consider \(G_{j}\) with \(l = \vert G_{j} \vert\), a partition of the observed data $ \mathbf{X}_n = \{X_1, \ldots, X_n\} $. For each \(G_{j}\), there exists a distribution \(Q_{n, m}^{(j)^*} \in \mathcal{Q}\) such that
\[
\mathrm{KL}\left(Q_{n, m}^{(j)^*}, \Pi\right) + m Q_{n, m}^{(j)^*}\left[\mathrm{KL}\left(\mathrm{P}_{0}^{(l)}, \mathrm{P}_{\theta}^{(l)}\right)\right] \leq \mathfrak{c}_5^l l \left(\eta_{l, m} + \zeta_{l, m}\right)^2,
\]
for some constant \(\mathfrak{c}_5^l > 0\).
\end{assumption}
The term $\eta_{l, m}$ denotes the approximation error of the model and $ \zeta_{l, m}$ the estimation error  for model $j$. 
Consequently the combined term, or \textit{oracle rate},
\begin{equation} 
\label{oraclerate}
\varepsilon_l:=\left(\eta_{l, m}+\zeta_{l, m}\right). 
\end{equation} 

The following proposition shows that Assumption \ref{Prior-cond} allows us to control
the variational approximation gap $\mathrm{P}_{0}^{(l)}\left[\mathrm{KL}\left(\widehat{Q}_{n,m}^{(j)}, \Pi_{n,m}\left(\cdot \mid G_j\right)\right)\right]$  to the original adjusted subset-based posterior $\Pi_{n,m}\left(\cdot \mid G_j\right).$

\begin{proposition}
\label{prop1}
(Variational approximation gap). Let \(l > 0\). Consider \(G_{j}\) with \(l = \vert G_{j} \vert\), a partition of the full data \(X_1, \ldots, X_n\). For any $j\in\{1,...,m\}$, we have that
\[
\mathrm{P}_{0}^{(l)}\left[\mathrm{KL}\left(\widehat{Q}_{n,m}^{(j)}, \Pi_{n,m}\left(\cdot \mid G_j\right)\right)\right] \leq \inf_{Q \in \mathcal{Q}}\left\{\mathrm{KL}\left(Q, \Pi\right) + m Q\left[\mathrm{KL}\left(\mathrm{P}_{0}^{(l)}, \mathrm{P}_{\theta}^{(l)}\right)\right]\right\}.
\]
Further, suppose that Assumption \ref{Prior-cond} holds. Then
\begin{equation}
\label{Rate}
\mathrm{P}_{0}^{(l)}\left[\mathrm{KL}\left(\widehat{Q}_{n,m}^{(j)}, \Pi_{n,m}\left(\cdot \mid G_j\right)\right)\right] \leq \mathfrak{c}_5^l l \varepsilon_l^2
\end{equation}
holds for any $j\in\{1,...,m\}$.
\end{proposition}

Proposition \ref{prop1} establishes an upper bound on the variational approximation error between the estimated subset-based variational posterior $\widehat{Q}_{n, m}^{(j)}$ and the true posterior $\Pi_{n, m}\left(\cdot \mid G_j\right)$ under the data generating process $\mathrm{P}_{0}^{(l)}$, for each data subset. 
This result is significant as it underscores that, in line with the literature (see, for instance, \cite{zhang2020convergence} and \cite{ohn2024adaptive}), 
when this gap is  of the order $c_5^l l \varepsilon_l^2$, the subset-based variational posterior can achieve the same contraction rate, 
$\epsilon_l$, as the true posterior, facilitating consistency in posterior estimation via classical change-of-measure arguments. This is shown in Theorem \ref{mainT-1} below.

In preparation for formally establishing this key concentration result, we first revisit   Theorem 1 in \cite{wong1995probability}, which plays a foundational role in our approach.  For a set \( A \subseteq \Theta \), the bracketing number \( N_{[]}(u, A, d) \) is associated with the family \( \{p_\theta, \theta \in A\} \), and is computed with respect to the distance 
\[
d(l, u) := \int_{\mathbb{R}^D} (\sqrt{l(x)} - \sqrt{u(x)})^2 \, dx.
\]
The \textit{bracketing entropy}, denoted \( H_{[]}(u; A) \), is then given by
\[
H_{[]}(u; A) := \log N_{[]}(u, A, d).
\]
Additionally, we denote the “Hellinger ball” of radius \( r \) centered at \( \theta_0 \) as 
\[
B(\theta_0, r) := \{\theta \in \Theta : h(P_\theta, P_{\theta_0}) \leq r\},
\]
where \( h(\cdot, \cdot) \) denotes the Hellinger distance. With these definitions in place, we now state the main result from \cite{wong1995probability}.
\begin{theorem}
\label{th:wong}
For a given constant \( \zeta > 0 \), there exist constants \( c_j \) for \( j = 1, \dots, 4 \) such that if 
\[
\int_{\zeta^2 / 2^8}^{\sqrt{2} \zeta} H_{[]}^{1 / 2} \left(\frac{u}{c_3}; B(\theta_0, \zeta \sqrt{2})\right) \, du \leq c_4 \sqrt{l} \zeta^2,
\]
then the probability bound
\[
\Pr\left(\sup_{\theta: h(P_\theta, P_0) \geq \zeta} \prod_{j=1}^l \frac{p_\theta}{p_0}(X_j) \geq e^{-c_1 l \zeta^2}\right) \leq 4 e^{-c_2 l \zeta^2}
\]
holds. In particular, these constants may be set as \( c_1 = \frac{1}{24} \), \( c_2 = \frac{4}{27} \cdot \frac{1}{1926} \), \( c_3 = 10 \), and \( c_4 = \frac{(2 / 3)^{5 / 2}}{512} \).
\end{theorem}
Theorem \ref{th:wong} provides a bound on the probability that the likelihood ratio deviates significantly from an exponentially small value for values of \( \theta \) lying outside a Hellinger ball centered at the true parameter. Specifically, Theorem \ref{th:wong} shows that the supremum of this likelihood ratio, taken over the set \( \{\theta : h(P_\theta, P_0) \geq \zeta\} \), is exponentially small with high probability. This bound, which decays at a rate proportional to \( l \zeta^2 \), effectively limits the probability of substantial deviations from the true density \( p_0 \) when \( \theta \) is sufficiently far from \( \theta_0 \).

In common parametric scenarios where \( \Theta \subseteq \mathbb{R}^p \), the bracketing entropy \( H_{[]}(u; B(\theta_0, r)) \) often satisfies the bound \( H_{[]}(u; B(\theta_0, r)) \leq C_1 \log (C_2 r / u) \), making it possible to select a minimal value for \( \zeta \) that satisfies the conditions of Theorem \ref{th:wong} with order \( \zeta \approx \sqrt{\frac{1}{l}} \). In particular, it is easy to check  via Theorem 2.7.11 of \cite{Vaart1996Weak-convergenc00}, that this is the case when the followings hold:
\begin{itemize}
    \item \textbf{Local Lower Bound}: There exists \( r_0 > 0 \) such that
    \[
    h(P_\theta, P_{\theta_0}) \geq K_1 \|\theta - \theta_0\|_2
    \]
    holds whenever \( h(P_\theta, P_{\theta_0}) \leq r_0 \).

    \item \textbf{Local Lipschitz Condition}: There exists \( \alpha > 0 \) such that for any \( \theta_1, \theta_2 \in B(\theta_0, r_0) \),
    \[
    |p_{\theta_1}(x) - p_{\theta_2}(x)| \leq F(x) \|\theta_1 - \theta_2\|_2^\alpha,
    \]
    where \( \int_{\mathbb{R}^D} F(x) \, dx < \infty \).
\end{itemize}
In our setup, we utilize the result in Theorem \ref{th:wong} by applying it when \( \zeta = \zeta_{l, m} \), with \( \zeta = \zeta_{l, m} \) introduced in Assumption \ref{Prior-cond}. Here, \( \zeta_{l, m} \) represents the estimation error  specific to each subset of the data partition, as previously detailed. By ensuring that \( \zeta_{l, m} \) meets a minimum value, determined by the bracketing Hellinger entropy, we confirm that each subset-based posterior \( \widehat{Q}_{n, m}^{(j)} \) remains concentrated around the true density \( p_0 \), even in partitioned data settings. This requirement is formalized in the following assumption.
\begin{assumption}
\label{test-cond}
    Consider the partition \( \{G_{j}\}_{j=1}^m \) of the sample $\{X_1,\ldots,X_n\}$ defined in Section \ref{SectionPartiotioning}, with \( l = \vert G_{j} \vert \). Assume that the conditions of Theorem \ref{th:wong} hold with \( \zeta := \zeta_{l,m} \) and subset-specific constants \( \mathfrak{c}_1^{l}, \mathfrak{c}_2^{l}, \mathfrak{c}_3^{l}, \) and \( \mathfrak{c}_4^{l} \).
\end{assumption}
By utilizing the entropy-based framework of Theorem \ref{th:wong} in Assumption \ref{test-cond}, we effectively restrict the likelihood of large deviations from \( p_0 \), 
which allows us to arrive at Theorem \ref{mainT-1}, establishing  the subset-based variational posterior contraction property. 
\begin{theorem}
\label{mainT-1}
 Let  $X_1,\ldots,X_n$ be i.i.d sampled from $P_{0}$. Consider  
 $\{G_j\}_{j=1}^m$ the partition defined in Section \ref{SectionPartiotioning}. Assume the Oracle rate, defined in Equation (\ref{oraclerate}), satisfies $l\eps_l^2\ge 1,$ where $l=\vert G_j\vert=\lfloor \frac{n}{m}\rfloor.$ Moreover, 
suppose Assumption 
\ref{Prior-cond}
and \ref{test-cond} hold.
Then, there exists a sufficiently large positive constant $R$ such that, for any $j\in\{1,...,m\}$ 
\begin{align}
\label{WeakCon-individual-1}
    &\mathbb{E}_0\Big(\widehat{Q}_{n,m}^{(j)}\left(\rho(\theta,\theta_0)\ge R\varepsilon_l\right)\Big)
    \le \exp(-\frac{\mathfrak{c}_1^l}{2}R^2l\varepsilon_l^2)+4\exp(-(\mathfrak{c}_2^l)^2R^2l\varepsilon_l^2)+\frac{1}{l\varepsilon_l^2}+3\mathfrak{c}_5^{l}\frac{1}{m}.
\end{align}
\end{theorem}
Theorem \ref{mainT-1} establishes that the subset-based variational posterior \( \widehat{Q}_{n,m}^{(j)} \) concentrates around the true parameter \( \theta_0 \) with high probability. Overall, this result ensures that the subset-based variational posterior retains concentration properties similar to the full posterior.


\subsection{Aggregation Step}
\label{medianVP}


In this section, we define and discuss two robust aggregation methods for combining subset-based variational posteriors defined Section \ref{Variationa-approx}: the Wasserstein geometric median ($Q^*_{\text{Geo}}$) and the Wasserstein metric median ($Q^*_{\text{Met}}$). This leads to our VM-posterior that is probably resistant to outliers, leveraging properties of the Wasserstein distance for robust and meaningful posterior aggregation while enabling practical computation.

\begin{definition}
\label{WassersteinDis}
Let $\mu_1$ and $\mu_2$ be any Borel probability measures  on the parameter space \((\Theta,\rho)\), with $\rho$ defined in Equation (\ref{therho}). The Wasserstein distance between \(\mu_1\) and \(\mu_2\) is defined as
$$
d_{W_{1, \rho}}(\mu_1, \mu_2)=\inf _{\gamma \in \Pi(\mu_1, \mu_2)} \int_{\Theta \times \Theta}\rho(x,y) d \boldsymbol{\gamma}(x, y),
$$
where \(\Pi(\mu_1, \mu_2)\) is the collection of all joint probability measures on \(\Theta \times \Theta\) with \(\mu_1\) and \(\mu_1\) as marginals. Specifically, for all subsets \(U \subset \Theta\), we have \(\gamma\left(U \times \Theta\right)=\mu_1(U)\) and \(\gamma\left(\Theta \times U\right)=\mu_2(U)\).
\end{definition}
The flexibility of the Wasserstein distance in handling probability distributions with potentially non-overlapping support makes it particularly well-suited for aggregating posteriors derived from distinct data partitions, where subtle yet important differences between subsets may exist.

In the following, to obtain a robust aggregation of  the collection of subset-based variational posteriors $\widehat{Q}_{n, m}^{(1)}, \ldots, \widehat{Q}_{n, m}^{(m)}$, we define the \textit{Wasserstein geometric median}, $Q_{\mathrm{Geo},}^*$, as follows,
\begin{equation}
    \label{geomed}
    Q^*_{\text{Geo}} =\operatorname{med}_{\text{g}}\left(\widehat{Q}_{n, m}^{(1)}, \ldots, \widehat{Q}_{n, m}^{(m)}\right)= \argmin_{Q \in \mathcal{Q}} \sum_{j=1}^m d_{W_{1, \rho}}(Q, \widehat{Q}_{n,m}^{(j)}).
\end{equation}
The Wasserstein geometric median $Q_{\mathrm{Geo}}^*$, inspired by the concept introduced in \cite{small1990survey}, can be seen as a natural extension of the univariate median to the space of probability measures over the parameter space $\Theta$. In the univariate case, the median minimizes the sum of distances from itself to all other points, providing a central location that is robust to outliers. Similarly, $Q_{\text {Geo }}^*$ is the measure that minimizes the sum of Wasserstein distances to all subset-based variational posteriors, resulting in a robust central representative of the distribution of subset-based variational posteriors. This formulation generalizes the robustness properties of the univariate median to higher-dimensional and probabilistic settings, allowing $Q_{\text {Geo }}^*$ to 
to resist skewed or extreme outliers.

Another useful generalization of the univariate median is the Wasserstein metric median, which adapts the broader notion of metric medians studied in \cite{JMLR:v17:14-273} to the context of Wasserstein distances and subset-based variational posteriors. Define $B_*$ to be the $d_{W_{1, \rho}}$-ball of minimal radius such that it is centered at one of $\left\{\widehat{Q}_{n, m}^{(1)}, \ldots, \widehat{Q}_{n, m}^{(m)}\right\}$ and contains at least half of these points. Then the \textit{Wasserstein metric median} of $\widehat{Q}_{n, m}^{(1)}, \ldots, \widehat{Q}_{n, m}^{(m)}$ is the center of $B_*$. In other words, let
$$
\begin{gathered}
\varepsilon_*:=\inf \{\varepsilon>0: \exists j=j(\varepsilon) \in\{1, \ldots, m\} \text { and } I(j) \subset\{1, \ldots, m\} \text { such that } \\
\left.|I(j)|>\frac{m}{2} \text { and } \forall i \in I(j), d_{W_{1, \rho}}\left(\widehat{Q}_{n, m}^{(i)}, \widehat{Q}_{n, m}^{(j)}\right) \leq 2 \varepsilon\right\}
\end{gathered}
$$
$j_*:=j\left(\varepsilon_*\right)$, where ties are broken arbitrarily, and set
\begin{equation}
\label{MetricMedian}
Q^*_{\text{Met}}=\operatorname{med}_0\left(\widehat{Q}_{n, m}^{(1)}, \ldots, \widehat{Q}_{n, m}^{(m)}\right):=\widehat{Q}_{n, m}^{(j_*)}.
\end{equation}
 Note that $Q^*_{\text{Geo}}$ and $Q^*_{\text{Met}}$ are always probability measures. Indeed, we can demonstrate that there exists $\alpha_1 \geq$ $0, \ldots, \alpha_m \geq 0$, $\sum_{j=1}^m \alpha_j=1$ such that $Q^*_{\text{Geo}}=\sum_{j=1}^m \alpha_j \widehat{Q}_{n, m}^{(j)}$, and $Q^*_{\text{Met}} \in\left\{\widehat{Q}_{n, m}^{(1)}, \ldots, \widehat{Q}_{n, m}^{(m)}\right\}$ by definition.

The metric median $Q_{\text {Met}}^*$, previously studied in \cite{JMLR:v17:14-273}, offers a robust aggregation of the subset-based variational posteriors by centering the posterior estimate within the smallest Wasserstein distance ball that contains at least half of the subset-based variational posteriors. By focusing on the subset-based variational posterior closest to the majority of others, this approach reduces the influence of extreme or outlying subset-based variational posteriors, making it particularly useful in situations where there are substantial variations between data partitions. Unlike the Wasserstein geometric median, which minimizes the overall sum of Wasserstein distances, the Wasserstein metric median is defined by proximity to the subset-based variational posteriors and effectively resists the impact of isolated but extreme variations. 


\section{Robustness of the VM-Posterior }
\label{robustness}

In this section, we establish theoretical results that demonstrate the robustness of the VM-Posterior approach, introduced in Section \ref{VM-approach-def}. This robustness is shown by analyzing both the Wasserstein geometric median and the Wasserstein metric median, defined in Section \ref{medianVP}, which together provide the foundation for robustly aggregating subset-based variational posteriors in the VM-Posterior construction.


The robustness of the Wasserstein geometric median and Wasserstein metric median can be precisely quantified through their concentration properties around the true Dirac measure, represented by the Dirac measure \( \delta_0 = \delta_{\theta_0} \), under potential contamination. Both methods possess the desirable property of transforming a collection of independent, “weakly concentrated” estimators into a single estimator with markedly stronger concentration properties, effectively mitigating the influence of outliers or extreme subset deviations. We state this formally in the following theorems.

\begin{theorem}
\label{thm-rob-gm}
Consider the disjoint subsets \( G_1, \dots, G_m \) as defined in Section \ref{SectionPartiotioning}. Let \break
\( \widehat{Q}_{n,m}^{(1)}, \dots, \widehat{Q}_{n,m}^{(m)} \) denote the subset-based variational posteriors defined in Section \ref{Variationa-approx}. Denote by $\kappa$  a constant satisfying $0 \leq \kappa<\frac{1}{3}$. Suppose \(\epsilon > 0\) is such that, for each \( j \) in the range \(1 \leq j \leq \lfloor(1 - \kappa) m\rfloor + 1\), the inequality
\begin{equation}
\label{WeakCon1}
\Pr\big(d_{W_{1,\rho}}(\widehat{Q}_{n,m}^{(j)}, \delta_0) > \epsilon\big) \leq \frac{1}{7}
\end{equation}
holds.
Furthermore, let \( Q^*_{\text{Geo}} \) denote the Wasserstein geometric median as defined in Equation (\ref{geomed}). Then, the following is satisfied,
\begin{align}
\label{rbust-geomed}
\Pr\big(d_{W_{1,\rho}}(Q^*_{\text{Geo}}, \delta_0) > 1.52\epsilon\big) \leq \left[ e^{(1 - \kappa) \psi\left( \frac{3/7 - \kappa}{1 - \kappa}, \frac{1}{7} \right)} \right]^{-m},
\end{align}
where the function \( \psi(\alpha, q) \) is given by,
\begin{equation}
\label{Psi-eq}
    \psi(\alpha, q) = (1 - \alpha) \log \frac{1 - \alpha}{1 - q} + \alpha \log \frac{\alpha}{q}.
\end{equation}
\end{theorem}

Theorem \ref{thm-rob-gm} implies that the concentration of the Wasserstein geometric median \( Q^*_{\text{Geo}} \) of independent estimators around the ``true" parameter \( \theta_0 \) 
improves geometrically with the number \( m \) of such estimators. Additionally, the estimation rate remains preserved up to a constant factor.
The strong concentration of 
\( Q^*_{\text{Geo}} \) in Equation (\ref{rbust-geomed})
  significantly enhances the weak concentration observed for individual subset-based variational posteriors in Equation (\ref{WeakCon1}). This improvement underscores the robustness of 
\( Q^*_{\text{Geo}} \)
  in aggregating the information across subsets while maintaining a high degree of concentration around the true parameter.

The parameter \( \kappa \) plays a crucial role in maintaining robustness to outliers. Specifically, if the initial sample contains up to \( \lfloor \kappa m \rfloor \) outliers (potentially of arbitrary nature), then no more than \( \lfloor \kappa m \rfloor \) of the subset-based posteriors \( \widehat{Q}_{n,m}^{(j)} \) can be impacted. Despite this, the Wasserstein geometric median remains close to the ``true" delta measure $\delta_0$ with high probability. This theorem demonstrates that even with some contamination among subsets, the Wasserstein geometric median retains its concentration properties, providing a measure of ``robustness" that tolerates outliers while preserving a strong alignment with the ``true" delta measure $\delta_0$.

To further clarify, suppose \( \widehat{Q}_{n,m}^{(1)}, \dots, \widehat{Q}_{n,m}^{(m)} \) are consistent estimators of \( Q_0 \) based on disjoint samples of size \( n / m \). As \( \frac{n}{m} \rightarrow \infty \), \( \frac{\kappa m}{n} \rightarrow 0 \), ensuring that the estimator \( Q^*_{\text{Geo}} \) remains consistent despite some contamination, handling a number of outliers that scales as \( o(n) \). This result matches the best-case scenario 
for outlier robustness without imposing additional restrictions on the distribution or nature of the outliers. Furthermore, because the Wasserstein geometric median resides in the convex hull of the subset-based posteriors, it effectively "downweights" outlier observations, making the approach practical for large-scale settings.

\begin{theorem}
\label{thm-rob-mm}
Consider the disjoint subsets \( G_1, \dots, G_m \) as defined in Section \ref{SectionPartiotioning}. Let \break \( \widehat{Q}_{n,m}^{(1)}, \dots, \widehat{Q}_{n,m}^{(m)} \) denote the subset-based variational posteriors defined in Section \ref{Variationa-approx}. Denote by $\kappa$  a constant satisfying $0 \leq \kappa<\frac{1}{3}$. Suppose \(\epsilon > 0\) is such that, for each \( j \) in the range \(1 \leq j \leq \lfloor(1 - \kappa) m\rfloor + 1\), the inequality
\begin{equation}
\label{WeakCon2}
\Pr\big(d_{W_{1,\rho}}(\widehat{Q}_{n,m}^{(j)}, \delta_0) > \epsilon\big) \leq \frac{1}{4}
\end{equation}
holds.
Furthermore, let \( Q^*_{\text{Met}} \) denote the Wasserstein metric median as defined in Equation (\ref{MetricMedian}). Then, the following is satisfied,
\begin{align}
\label{rbust-metmed}
\Pr\big(d_{W_{1,\rho}}(Q^*_{\text{Met}}, \delta_0) > 3\epsilon\big) \leq \left[ e^{(1 - \kappa) \psi\left( \frac{1/2 - \kappa}{1 - \kappa}, \frac{1}{4} \right)} \right]^{-m}.
\end{align}
\end{theorem}

Theorem \ref{thm-rob-mm} provides a robustness guarantee for the Wasserstein metric median \( Q^*_{\text{Met}} \), analogous to the guarantee for the Wasserstein geometric median in Theorem \ref{thm-rob-gm}. Defined by the smallest Wasserstein distance ball containing at least half of the subset-based variational posteriors, the metric median requires only pairwise distance information, making it especially practical in high-dimensional settings. 

As with the geometric median, the parameter \( \kappa \) limits the impact of up to \( \lfloor \kappa m \rfloor \) outliers, ensuring that \( Q^*_{\text{Met}} \) remains concentrated around the true delta measure \( \delta_0 \) with high probability. 

Overall, these two theorems illustrate an essential property of the median aggregation step, whereby a collection of weakly concentrated estimators is transformed into a single estimator with markedly stronger concentration around the true parameter \( \theta_0 \). Both the Wasserstein geometric median and the Wasserstein metric median provide robust methods for aggregating subset-based variational posteriors, even under contaminations of arbitrary nature. 

Next, we establish the weak concentration properties of each subset-based variational posterior, a critical component in proving the robustness of the Wasserstein geometric and metric medians as discussed earlier. Specifically, we develop concentration properties that satisfy the weak concentrations established in Equation (\ref{WeakCon1}) and Equation (\ref{WeakCon2}). This is provided in Theorem \ref{mainT} which can be viewed as an adaptation of Theorem \ref{mainT-1}, applied to the Wasserstein distance \(d_{W_{1,\rho}}(\widehat{Q}_{n,m}^{(j)}, \delta_0)\)  rather than the closely related contraction rate concept of the posterior distribution itself. Here, the Wasserstein distance is specified in Definition \ref{WassersteinDis} and evaluated with respect to the \textit{Hellinger metric} \( \rho(\cdot, \cdot) \), which introduced in Equation (\ref{therho}).


\begin{theorem}
\label{mainT}
 Let  $X_1,\ldots,X_n$ be i.i.d sampled from $P_{0}$. Consider  
 $\{G_j\}_{j=1}^m$ the partition defined in Section \ref{SectionPartiotioning}. Assume the Oracle rate, defined in Equation (\ref{oraclerate}), satisfies $l\eps_l^2\ge 1,$ where $l=\vert G_j\vert=\lfloor \frac{n}{m}\rfloor.$ Moreover, 
suppose Assumption 
\ref{Prior-cond}
and \ref{test-cond} hold.
Then, there exists a sufficiently large positive constant $R$ such that, for any $j\in\{1,...,m\}$ 
\begin{align}
\label{WeakCon-individual}
    &\Pr\Big(d_{W_{1,\rho}}(\widehat{Q}_{n,m}^{(j)},\delta_0)\ge 2R\varepsilon_l+\exp(-ml\varepsilon_l^2)\Big)
    \nonumber
    \\
    \le& \exp(-\frac{\mathfrak{c}_1^l}{2}R^2l\varepsilon_l^2)+4\exp(-(\mathfrak{c}_2^l)^2R^2l\varepsilon_l^2)+\frac{1}{l\varepsilon_l^2}+2\mathfrak{c}_5^{l}\frac{1}{m}+\frac{\mathfrak{c}_5^{l}}{Rm\eps_l}.
\end{align}
\end{theorem}
This theorem establishes a probability bound for the Wasserstein distance between each subset-based variational posterior \( \widehat{Q}_{n, m}^{(j)} \) and the Dirac measure \( \delta_0 \), representing concentration around the true parameter \( \theta_0 \). The bound shows that the Wasserstein distance is constrained by both the oracle rate \( \varepsilon_l \) and a term that decays exponentially with respect to \( m \) and \( l \). This result highlights that subset-based variational posteriors, each operating on a reduced sample size, still achieve concentration properties that are reflective of the full posterior contraction rate, \( \varepsilon_l \). Additionally, the probability bound comprises terms that diminish as \( m \) increases, supporting the inference that as we partition into more subsets, weak concentration remains achievable.

Theorem \ref{mainT} directly connects to the robustness properties of the Wasserstein geometric median aggregation, \( Q_{\mathrm{Geo}}^* \), presented in Corollary \ref{GeomMed-Theorem}.
\begin{corollary}
\label{GeomMed-Theorem}
 Let  $X_1,\ldots,X_n$ be i.i.d sampled from $P_{0}$. Assume $0\le\kappa<\frac{1}{3}$ and that the data set $\{X_1,\ldots,X_n\}$ contains at most $\lfloor \kappa m\rfloor$ outliers. Consider  
 $\{G_j\}_{j=1}^m$ the partition defined in Section \ref{SectionPartiotioning}.  
Suppose Assumption 
\ref{Prior-cond}
and \ref{test-cond} hold. Moreover, let the Oracle rate, defined in Equation (\ref{oraclerate}), satisfies $l\eps_l^2\ge 1,$ and for any $1\le j\le\lfloor(1-\kappa)m\rfloor+1$ we have that
\begin{align*}
     \exp(-\frac{\mathfrak{c}_1^l}{2}R^2l\varepsilon_l^2)+4\exp(-(\mathfrak{c}_2^l)^2R^2l\varepsilon_l^2)+\frac{1}{l\varepsilon_l^2}+2\mathfrak{c}_5^{l}\frac{1}{m}+\frac{\mathfrak{c}_5^{l}}{Rm\eps_l}\le \frac{1}{7},
\end{align*}
where $l=\vert G_j\vert=\lfloor \frac{n}{m}\rfloor$,
and $R$ is a sufficiently large positive constant.
Then,
\begin{align*}
    \Pr(d_{W_{1,\rho}}(Q^*_{\text{Geo}}, \delta_0) \ge 1.52(2R\varepsilon_l+\exp(-ml\varepsilon_l^2)))\le\left[ e^{(1 - \kappa) \psi\left( \frac{3/7 - \kappa}{1 - \kappa}, \frac{1}{7} \right)} \right]^{-m},\end{align*}
    where the function $\psi$ is defined in Equation (\ref{Psi-eq}).
\end{corollary}
Corollary \ref{GeomMed-Theorem} is a direct consequence of Theorem \ref{thm-rob-gm} and Theorem \ref{mainT}.  Note that the weak concentration assumption in Equation (\ref{WeakCon1}) is
implied by Equation (\ref{WeakCon-individual}). It is easy to see that a similar statement holds for $Q^*_{\text{Met}}$ by applying Theorem \ref{thm-rob-mm} and Theorem \ref{mainT}.

These results confirm that, given sufficient sample size \(l\) and partition count \(m\), the Wasserstein geometric median and Wasserstein metric median of the subset-based variational posteriors will remain concentrated around \( \delta_0 \) with high probability. The probability bound demonstrates that the aggregated posterior \( Q^*_{\text{Geo}} \) and \( Q^*_{\text{Met}} \) are robust to potential contamination and computationally efficient, preserving the contraction rate even as \( m \) grows. This provides substantial support for the effectiveness of median-based aggregation in achieving a balance between computational efficiency and robustness in Bayesian inference, especially in large-scale settings. Notably, these results yield an exponential improvement in concentration compared to Theorem \ref{mainT}, underscoring the advantages of the VM-Posterior approach.

\section{Asymptotic Normality and Confidence Bounds for VM-Posterior}
\label{sec-Asympnorm}
In this section, we establish the asymptotic properties of the VM-Posterior, particularly its convergence to a normal distribution under the total variation (TV) distance. Additionally, we provide finite-sample confidence bounds for the estimated parameter, showcasing the robustness and reliability of the VM-Posterior in structured settings. 

To facilitate this analysis, we focus on the asymptotic behavior of the Wasserstein metric median \( Q_{\text{Met}}^* \) by first specifying the variational family used in its construction. We consider two distinct families: the mean-field family,
\[
\mathcal{Q}_{\text{MF}} = \left\{ \prod_{j=1}^d q_j(\theta_j) \right\},
\]
and the Gaussian family,
\[
\mathcal{Q}_{\text{GG}} = \left\{ N(\mu, \Sigma): \mu \in \mathbb{R}^D, \Sigma \in \mathbb{R}^{D \times D} \text{ is positive definite} \right\}.
\]
This distinction enables us to analyze the Wasserstein metric medians \( Q_{\text{Met,MF}}^* \) and \( Q_{\text{Met,GG}}^* \), where \( Q_{\text{Met}}^* \) is defined using either the mean-field or Gaussian family. For more information on these variational family specifications, refer to Section \ref{Variationa-approx}.

In the parametric setting, where the parameter space \( \Theta \subseteq \mathbb{R}^p \), we demonstrate that as the sample size \( n \) grows, the VM-Posterior \( Q_{\text{Met}}^* \) converges to a normal distribution centered at a robust estimator \( \theta^* \) of the true parameter \( \theta_0 \). This convergence under the TV distance illustrates the VM-Posterior’s behavior in finite-sample conditions and provides insight into its stability and robustness in the parametric case.

Furthermore, we establish a finite-sample bound for the estimator \( \theta^* \), confirming that it represents the center of a high-confidence region, with a convergence rate comparable to classical posterior distributions. These results underscore the reliability of the VM-Posterior in producing well-calibrated, robust estimates, even in large-scale settings and in the presence of potential contamination across data subsets.

To set the stage, we begin by analyzing each subset-based variational posterior \( \{Q_{n,m}^{(j)}\}_{j=1}^m \) individually. For \( \theta \in \Theta \), let
\[
I(\theta) := \mathbb{E}_{\theta_0}\left[\frac{\partial}{\partial \theta} \log p_\theta(X) \left( \frac{\partial}{\partial \theta} \log p_\theta(X) \right)^T\right]
\]
be the Fisher information matrix, which we assume is well-defined. We say that the family \( \{P_\theta : \theta \in \Theta\} \) is differentiable in quadratic mean (see Chapter 7 in \cite{van2000asymptotic} for details) if there exists a function \( \dot{\ell}_{\theta_0} : \mathbb{R}^D \rightarrow \mathbb{R}^p \) such that
\[
\int_{\mathbb{R}^D} \left( \sqrt{p_{\theta_0 + h}} - \sqrt{p_{\theta_0}} - \frac{1}{2} h^T \dot{\ell}_{\theta_0} \sqrt{p_{\theta_0}} \right)^2 = o\left(\|h\|_2^2\right)
\]
as \( h \rightarrow 0 \). Generally, we have \( \dot{\ell}_\theta(x) = \frac{\partial}{\partial \theta} \log p_\theta(x) \). Using this framework, we define
\[
\Delta_{l, \theta_0} := \frac{1}{\sqrt{l}} \sum_{j=1}^l I^{-1}(\theta_0) \dot{\ell}_{\theta_0}(X_j),
\]
which will be instrumental in analyzing the distributional behavior of each subset-based variational posterior.

We now formalize these properties with the following theorem, which provides a convergence result for each subset-based variational posterior under the Gaussian and mean-field families.
\begin{theorem}\label{theorem1BVM}
 Let  $X_1,\ldots,X_n$ be i.i.d sampled from $P_{0}$. Consider  
 $\{G_j\}_{j=1}^m$ the partition defined in Section \ref{SectionPartiotioning}. Denote by 
$\{Q_{n,m}^{(j),\text{GG}}\}_{j=1}^m$
  the subset-based variational posteriors from Section \ref{Variationa-approx}, constructed using the Gaussian variational family 
$\mathcal{Q}_{\text{GG}}$.
Then, for any $1\le j\le m,$ we have that
$$
\left\|Q_{n,m}^{(j),\text{GG}}(\cdot)-\mathcal{N}\left(\cdot ; \theta_0+\frac{\Delta_{l, \theta_0}}{\sqrt{l}}, \frac{I^{-1}(\theta_0)}{lm}\right)\right\|_{\mathrm{TV}}\xrightarrow{P_{\theta_0}} 0 .
$$
The same result holds for \( \{Q_{n,m}^{(j),\text{MF}}\}_{j=1}^m \) where \( \{Q_{n,m}^{(j),\text{MF}}\}_{j=1}^m \) is constructed using the mean-field variational family \( \mathcal{Q}_{\text{MF}} \). In this case, \( I^{\prime-1}(\theta_0) \) replaces \( I^{-1}(\theta_0) \), where \( I^{\prime-1}(\theta_0) \) is diagonal and shares the same diagonal entries as \( I^{-1}(\theta_0) \).
\end{theorem}
 This theorem follows directly from Corollary 7 in \cite{wang2019frequentist} for the convergence of \( \{Q_{n,m}^{(j),\text{GG}}\}_{j=1}^m \) and Theorem 5 in \cite{wang2019frequentist} for the corresponding result on \( \{Q_{n,m}^{(j),\text{MF}}\}_{j=1}^m \), which lays the foundation for analyzing the aggregated posterior \( Q_{\text{Met}}^* \).

To proceed with the analysis of the aggregated posterior \( Q_{\text{Met}}^* \), we introduce a uniform integrability assumption for the collections \( \{Q_{n,m}^{(j),\text{GG}}\}_{j=1}^{m} \) and \( \{Q_{n,m}^{(j),\text{MF}}\}_{j=1}^{m} \). This condition ensures bounded second moments for each subset-based variational posterior, a necessary foundation for deriving asymptotic results on \( Q_{\text{Met}}^* \).
\begin{assumption}\label{assumption1BVM}
We assume uniform integrability for the collections \( \{Q_{n,m}^{(j),\text{GG}}\}_{j=1}^{m} \) and \break 
\( \{Q_{n,m}^{(j),\text{MF}}\}_{j=1}^{m} \). Specifically, this requires that
    $$
\sup_{1 \leq j \leq m} \int_{\mathbb{R}^p} \|\theta\|_2^2 \, d Q_{n,m}^{(j),\text{GG}}(\theta) < \infty, \quad \text{and} \quad \sup_{1 \leq j \leq m} \int_{\mathbb{R}^p} \|\theta\|_2^2 \, d Q_{n,m}^{(j),\text{MF}}(\theta) < \infty.
$$
\end{assumption}
With Assumption \ref{assumption1BVM} in place, we are now equipped to establish the asymptotic normality and finite-sample confidence bounds for the Wasserstein metric median \( Q_{\text{Met}}^* \) of the VM-Posterior.
\begin{theorem}
    \label{BVM-theorem}
    Let \( X_1,\ldots,X_n \) be i.i.d. samples from \( P_{0} \). 

    \begin{enumerate}[label=\alph*)]
        \item For any fixed \( m \geq 1 \), let \( \{G_j\}_{j=1}^m \) be the partition defined in Section \ref{SectionPartiotioning}. Suppose Assumption \ref{assumption1BVM} holds. Then, we have
        \begin{equation}
            \left\| Q^*_{\text{Met,GG}} - \mathcal{N}\left(\cdot ; \theta^*_{\text{GG}}, \frac{1}{n} I^{-1}\left(\theta_0\right)\right) \right\|_{\mathrm{TV}} \rightarrow 0,
        \end{equation}
        in \( P_{\theta_0} \)-probability as \( n \rightarrow \infty \), where \( \theta^*_{\text{GG}} \) is the mean of \( Q^*_{\text{Met,GG}} \). A similar result holds for \( Q^*_{\text{Met,MF}} \), with mean \( \theta^*_{\text{MF}} \) and \( I^{\prime-1}\left(\theta_0\right) \) used in place of \( I^{-1}\left(\theta_0\right) \), where \( I^{\prime-1}\left(\theta_0\right) \) is a diagonal matrix matching the diagonal elements of \( I^{-1}\left(\theta_0\right) \).

        \item  Let \( \{G_j\}_{j=1}^m \) be the partition defined in Section \ref{SectionPartiotioning}. Assume \( 0 \leq \kappa < \frac{1}{3} \) and that the dataset \( \{X_1,\ldots,X_n\} \) contains at most \( \lfloor \kappa m \rfloor \) outliers. Suppose Assumptions \ref{Prior-cond} and \ref{test-cond} hold. Additionally, consider the Oracle rate defined in Equation (\ref{oraclerate}) satisfy \( l \varepsilon_l^2 \geq 1 \), and $\sqrt{n}\le m$,
        where \( l = \vert G_j \vert = \lfloor \frac{n}{m} \rfloor \) and \( R \) is a sufficiently large positive constant.
        
        Then,
        \begin{align*}
            \Pr\left(\left\|\theta^*_{\text{GG}} - \theta_0\right\|_2 \geq 3 \left( 2R \varepsilon_l + \exp(-ml \varepsilon_l^2) \right) \right) \leq \left[ e^{(1 - \kappa) \psi\left( \frac{1/2 - \kappa}{1 - \kappa}, \frac{1}{4} \right)} \right]^{-m},
        \end{align*}
        with the function \( \psi \) is defined in Equation (\ref{Psi-eq}). A corresponding result holds for \( \theta^*_{\text{MF}} \).
    \end{enumerate}
\end{theorem}
Theorem \ref{BVM-theorem} extends classical asymptotic results to the setting of the VM-Posterior, demonstrating that the Wasserstein metric median \( Q_{\text{Met}}^* \), constructed from subset-based variational posteriors, converges to a normal distribution centered at a robust estimator \( \theta^* \) of the true parameter \( \theta_0 \). This is achieved under both Gaussian and mean-field variational families, where the precision matrices \( I^{-1}(\theta_0) \) and \( I^{\prime-1}(\theta_0) \) capture the effective sample size across the data subsets. 

Part (a) of the theorem mirrors the classical Bernstein–von Mises (BvM) theorem, which traditionally establishes normality for Bayesian posteriors in large samples. However, Theorem \ref{BVM-theorem} is significant because it applies to the median-aggregated posterior in a variational context, providing a robust, subset-based approach that retains asymptotic normality even in the presence of limited contamination. The results here indicate that \( Q_{\text{Met}}^* \) achieves a convergence rate analogous to that expected from a full-sample posterior, thus demonstrating its stability and reliability under variational approximations.

Part (b) provides a finite-sample confidence bound, distinguishing this result from traditional asymptotic BvM settings by quantifying the robustness of the VM-Posterior under finite sample sizes and potential contamination among subsets. This added robustness suggests that \( Q_{\text{Met}}^* \) can serve as a reliable posterior estimate in large-scale settings, where data contamination or high computational demands may challenge conventional Bayesian posteriors. As such, Theorem \ref{BVM-theorem} reinforces the practical appeal of the VM-Posterior for real-world applications requiring scalable, robust Bayesian inference.

\section{Computation Algorithms} 
\label{AlgoSection} In this section, we delve into the computation details of the variational approximation algorithm and explain how to compute the VM-posterior. The code for this paper is available at \url{https://github.com/waterism211/variational-median.git}.

\subsection{Algorithm for the Variational Approximation}\label{sec-AlgoVI} Variational approximations offer various techniques to approximate complex posterior distributions. Below, we outline two prominent methods that can be applied within the variational framework.

\begin{itemize} \item \textbf{Mean-field Family:} The mean-field approximation simplifies computation by assuming that all variables are independent. We employ the widely-used Coordinate Ascent Variational Inference (CAVI) method as presented by \cite{Bishop2006Pattern-recogni00}. CAVI iteratively optimizes each variational parameter while holding others fixed, which yields an efficient and tractable solution in many cases.
\item \textbf{Gaussian Family with general covariance:} A  flexible approach is to approximate the posterior distribution with a Gaussian family with a general covariance structure. This method extends beyond the mean-field assumption, allowing for dependencies between variables. Inspired by \cite{mahdisoltani2021natural}, we implement Stochastic Variational Inference (SVI), where we approximate the posterior using stochastic optimization techniques. SVI is particularly useful in large datasets, as it optimizes variational parameters through mini-batch updates.
\end{itemize}
 
\subsection{Algorithm to Compute the 
Geometric Median respect to Wasserstein distance}\label{sec-algowass} 
Following the approach from \cite{gmmot2024}, we compute the geometric median of a set of Gaussian distributions by utilizing linear programming in optimal transport. Additionally, we employ Weiszfeld's algorithm for the geometric median among discrete distributions.

\subsubsection{Multivariate Gaussian Distributions}\label{sec-mgauss} In the case where the distributions $Q_j$, $j=1,\dots,m$, are multivariate Gaussians $N(\mu_j,\Sigma_j)$, the geometric median, denoted as $Q^*_{\text{Geo}}$, is computed by an 
iterative algorithm that updates both the mean and covariance parameters iteratively, as described in \cite{alvarez2016fixed}. 

Given covariance matrices $\Sigma_j$ for each Gaussian distribution, the algorithm updates the covariance matrix $S_n$ of the geometric median at each iteration $n$ as follows: 
\begin{equation} 
S_{n+1} = \left(\sum_{j=1}^m \frac{1}{m}\left(S_n^{1/2} \Sigma_j S_n^{1/2}\right)^{1/2}\right) 
\end{equation}
The mean of $Q^*_{\text{Geo}}$ is updated by computing the median of the means $\mu_j$, which provides a robust central tendency measure across the distributions. This is described in Algorithm \ref{alg1} below. 

\begin{algorithm}[H]
\caption{Geometry Median of Gaussian Distributions}
\label{alg1}
\SetKwInOut{KwIn}{Input}
\SetKwInOut{KwOut}{Output}
\KwIn{$m$ Gaussian distributions $Q_j$ with means $\mu_j$ and covariances $\Sigma_j$, initial covariance $S_0$,  number of iterations $N$}
\KwOut{Geometry median $Q^*_{\text{Geo}}$}
\textbf{Calculate Mean:}
    $ \mu \leftarrow \text{Median}(\mu_1, \ldots, \mu_m)$\\
\For{each iteration $n$}{
    \textbf{Update Covariance:} 
    $S_{n+1} \leftarrow \left(\sum_{j=1}^{m} \frac{1}{m} \left(S_n^{1/2} \Sigma_j S_n^{1/2}\right)^{1/2} \right)$\;

}
$Q^*_{\text{Geo}}$ will be a Gaussian distribution with  mean $\mu$ and covariance $S_N$.
\end{algorithm}

\subsubsection{Gaussian Mixture Models}\label{sec-gmm}
For Gaussian mixture models (GMMs), where each distribution \( Q_j \) can be represented as a weighted sum of Gaussian components, we write:
\[
Q_j = \sum_{k=1}^{K_j} \pi_j^k Q_j^k
\]
where each \( Q_j^k \) is a multivariate Gaussian \( N(\mu_j^k, \Sigma_j^k) \). The goal is to compute the geometric median \( Q^*_{\text{Geo}} \) over these mixture distributions. Since the median of GMMs are intractable (not necessarily a GMM), we will  find an approximation of it. The solution is given by \( Q^*_{\text{Geo}} = B \# \gamma^* \), where \( B: (\mathbb{R}^d)^m \to \mathbb{R}^d \) maps \( (x_1, \ldots, x_m) \) to \( \sum \lambda_m x_j \). Here, \( \gamma^* \) represents an optimal multi-marginal transport plan.

The multi-marginal optimal transport problem for the geometric median is given by:
\[
Q^*_{\text{Geo}}(Q_1, \dots, Q_m) := \inf_{\gamma \in \Pi(Q_1, \ldots, Q_m) \cap GMM_{md}(\infty)} \int_{\mathbb{R}^{dm}} c(x_1, \dots, x_{m}) \, d\gamma(x_1, \dots, x_{m})
\]
where the cost function \( c(x_0, \ldots, x_{m-1}) \) is defined as:
\[
c(x_1, \ldots, x_{m}) = \sum_{i=1}^{m} \frac{1}{m} \| x_i - B(x) \|^2 = \frac{1}{2m^2} \sum_{i, j=0}^{m-1} \| x_i - x_j \|^2
\]
and \( \Pi(Q_1, \ldots, Q_m) \) denotes the set of probability measures on \( (\mathbb{R}^d)^m \) with \( Q_1, Q_2, \ldots, Q_m \) as marginals, and \( GMM_{md}(\infty) \) refers to the Gaussian mixture class.

To discretize the multi-marginal problem:
The optimal solution \( \gamma^* \) can be expressed as a sum of the optimal transport plans between the Gaussian components of the GMMs:
\[
\gamma^* = \sum_{\substack{1 \leq k_1 \leq K_1 \\ 1 \leq k_{m} \leq K_{m}}} w_{k_1 k_2 \ldots k_{m}}^* \gamma_{k_1 k_2 \ldots k_{m}}^* ,
\]
where \( \gamma_{k_1 k_2 \ldots k_{m}}^* \) is the optimal multi-marginal plan between the Gaussian measures \( Q_1^{k_0}, \ldots, Q_m^{k_{m}} \). Furthermore, \( w^* \) is the solution of the discrete optimization problem:
\[
\min _{w \in \Pi\left(\pi_1, \ldots, \pi_{m}\right)} \sum_{k_1, \ldots, k_{m}}^{K_1, \ldots, K_{m}} w_{k_1 \ldots k_{m}} Q^*_{\text{Geo}}(Q_1,\dots,Q_m),
\]
where \( \Pi\left( \pi_1, \pi_2, \ldots, \pi_{m}\right) \) is the subset of tensors \( w \) in \( \mathcal{M}_{K_1, K_2, \ldots, K_{m}}\left(\mathbb{R}^{+}\right) \) with \( \pi_1, \pi_2, \ldots, \pi_{m} \) as discrete marginals.

In practice, we compute all the values \( Q^*_{\text{Geo}}(Q_1, \dots, Q_m) \) and the Gaussian plans \( \gamma_{k_1 k_2 \ldots k_{m}}^* \) between the components of the GMM. We solve for the weights \( w^* \) using a linear programming solver (e.g., \texttt{linprog} from Scipy).

\textbf{Compute \( Q^*_{\text{Geo}} \) from \( w^* \):} To solve the discrete multi-marginal problem, we use \texttt{linprog} from SciPy. The function \texttt{create\_cost\_matrix\_from\_gmm} in \cite{gmmot2024} takes a list of GMM parameters, a vector of weights, and an integer \( N \) (corresponding to the number of iterations in the function \texttt{GaussianMedianW2}), and outputs the cost matrix \( C \) and the list of all Gaussian geometry medians between the Gaussian components of the GMM. The function \texttt{solveMMOT} in \cite{gmmot2024} constructs matrices \( A \) and \( b \) to encode the constraints on \( w \) and uses \texttt{linprog} to solve:
\[
\inf _{A w=b}\langle C, w \rangle.
\]

The resulting weights correspond to the mixture components of our \( Q^*_{\text{Geo}} \). To summarize, the algorithm proceeds as follows:

\begin{algorithm}[H]
\caption{Geometric Median among Many Gaussian Mixtures}
\label{alg2}
\SetKwInOut{KwIn}{Input}
\SetKwInOut{KwOut}{Output}

\KwIn{$m$ Gaussian mixtures $Q_1, \ldots, Q_m$}
\KwOut{Geometric median $Q_{\text{Geo}}^*$}

\textbf{Create Cost Matrix:} \\
Represent $Q_j$ as a weighted sum of components: $Q_j = \sum_{k=1}^{K_j} \pi_j^k Q_j^k $. Use the function \texttt{create\_cost\_matrix\_from\_gmm} to compute the cost matrix $C$ and the list of Gaussian geometry median:
\[
C = Q^*_{\text{Geo}}(Q_1^{k_1}, \dots, Q_m^{k_m})
\]

\textbf{Set up Linear Program:} \\
Construct matrices $A$ and $b$ encoding the constraints on the weights $w$, where $w \in \Pi(\pi_1, \ldots, \pi_{m})$ and $A w = b$\;

\textbf{Solve MMOT Problem:} \\
Use \texttt{linprog} to solve the linear programming problem:
\[
\inf_{Aw =b} \langle C, w \rangle
\]
\end{algorithm}
\subsection{Algorithm for Computing Geometric Median of General Distributions}\label{sec-kerneldist}
\cite{minsker2017robust} proposed an algorithm for computing the median of any posterior distribution with respect to the Reproducing Kernel Hilbert Space (RKHS) distance, which corresponds to the median in Equation (\ref{geomed}). An R package for RKHS with a radial basis function (RBF) kernel, developed by \cite{SBmedian}, facilitates the implementation of this algorithm. The general procedure is outlined as follows:

\begin{algorithm}[H]
\caption{Geometric median of probability distributions via Weiszfeld's}
\label{alg3}
\SetKwInOut{KwIn}{Input}
\SetKwInOut{KwOut}{Output}

\KwIn{Discrete measures $Q_1, \ldots, Q_m$; The kernel $k(\cdot, \cdot): \mathbb{R}^p \times \mathbb{R}^p \mapsto \mathbb{R}$; Threshold $\varepsilon>0$}
\KwOut{Weights for the median respect to Discrete measure: $w_*:=\left(w_1^{(t+1)}, \ldots, w_m^{(t+1)}\right)$}

\textbf{Initialize:} Set $w_j^{(0)}:=\frac{1}{m}, j=1 \ldots m$; and $Q^{(0)}:=\frac{1}{m} \sum_{j=1}^m Q_j$.\\
Starting from $t=0$, for each $j=1, \ldots, m$:\\  \While{$\left\|Q^{(t+1)}-Q^{(t)}\right\|_{\mathcal{F}_k} \leq \varepsilon$}
{
        1. Update $w_j^{(t+1)}=\frac{\left\|Q^{(t)}-Q_j\right\|_{\mathcal{F}_k}^{-1}}{\sum_{i=1}^m\left\|Q^{(t)}-Q_i\right\|_{\mathcal{F}_k}^{-1}} ;$\\
        2. Update $Q_t^{(t+1)}=\sum_{j=1}^m w_j^{(t+1)} Q_j ;$
}
\end{algorithm}



\subsection{Likelihood Power Adjustment in Practice}
Since we performed data-splitting and computed the geometric median for our VM-Posterior, regularization of the covariance term is necessary. In the variational approximation algorithm, there are two ways to address this:

\begin{itemize}
    \item For the SVI algorithm, we optimize the Evidence Lower Bound (ELBO). To account for likelihood power adjustment, we optimize the following term:
    \begin{equation*}
        E_{q(\theta)}\{ m \cdot \log p(X|\theta) - \log \pi(\theta) - \log q(\theta)\}
    \end{equation*}
    where \( m \) is the number of sub-datasets into which we split the data.
    
    \item Alternatively, we can adjust the covariance term of the VM-Posterior directly by dividing the covariance by \( \sqrt{m} \).
\end{itemize}

\section{Numerical studies}\label{simu-data}
In this section, we compare the performance of our VM-Posterior with that of the variational posterior, M-posterior, and the original  posterior distribution across the following models:

\subsection{Multivariate Gaussian Models}\label{sec-gausssimu}
This section demonstrates with a simple example that the effect of the magnitude of an outlier on the variational posterior distribution of the mean parameter \( \mu \). We show that the VM-Posterior achieves robustness similar to the M-posterior in \cite{minsker2017robust}, with significantly faster computation. We consider the univariate Gaussian model \( X \sim N(\mu, 1) \).

Firstly, we simulated 25 data sets, each containing 100 observations. Each data set \( \bold{x}_i = (x_{i,1}, \dots, x_{i,100}) \) contained 99 independent observations from a standard Gaussian distribution \( x_{i,j} \sim N(\mu = 2, 1) \) for \( i = 1, \dots, 25 \) and \( j = 1, \dots, 99 \). The last entry in each data set \( x_{i,100} \) was an outlier, and its value increased linearly for \( i = 1, \dots, 15 \) with \( x_{i,100} = \max(\bold{x}_{i,1:99}) \). The index of the outlier was unknown to the estimation algorithm, and we assumed that the variance of observations was known.

For computation, we performed both standard variational Bayes using \texttt{cmstan} with 2 chains and 1000 samples, and VM-posterior  for the multivariate Gaussian model. We set the initial values to the true value of \( \mu \) and ran 1000 iterations. The VM-posterior approach proceeded as follows: For each data set \( \bold{x}_i \), we randomly separated it into 10 groups, \( G_1, \dots, G_{10} \). We then computed \( q(\cdot | G_1), \dots, q(\cdot | G_{10}) \) and used Algorithm \ref{alg1} to find the geometric median of the multivariate Gaussians.

This method was applied to each data set \( \bold{x}_1, \dots, \bold{x}_{15} \) to analyze the impact of the outlier magnitude. We replicated each outlier level 50 times to assess whether the posterior credible interval contained the true value of \( \mu \).
\begin{figure}[htbp!]
 \centering
 \begin{subfigure}[b]{0.3\textwidth}
     \centering
     \includegraphics[width=\textwidth]{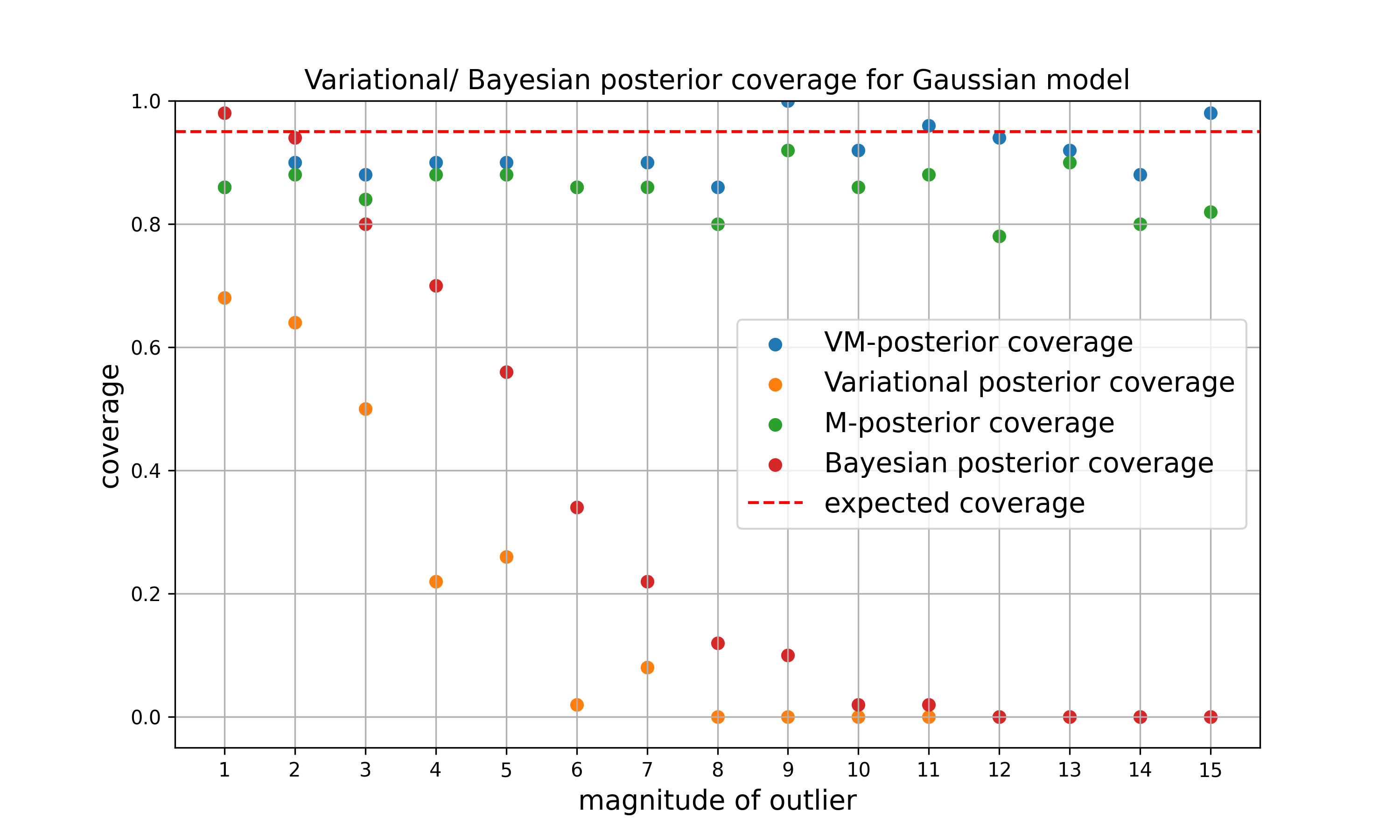}
     \caption{$\alpha$=0.05}
     \label{fig:95cov}

 \end{subfigure}
 \hfill
 \begin{subfigure}[b]{0.3\textwidth}
     \centering
     \includegraphics[width=\textwidth]{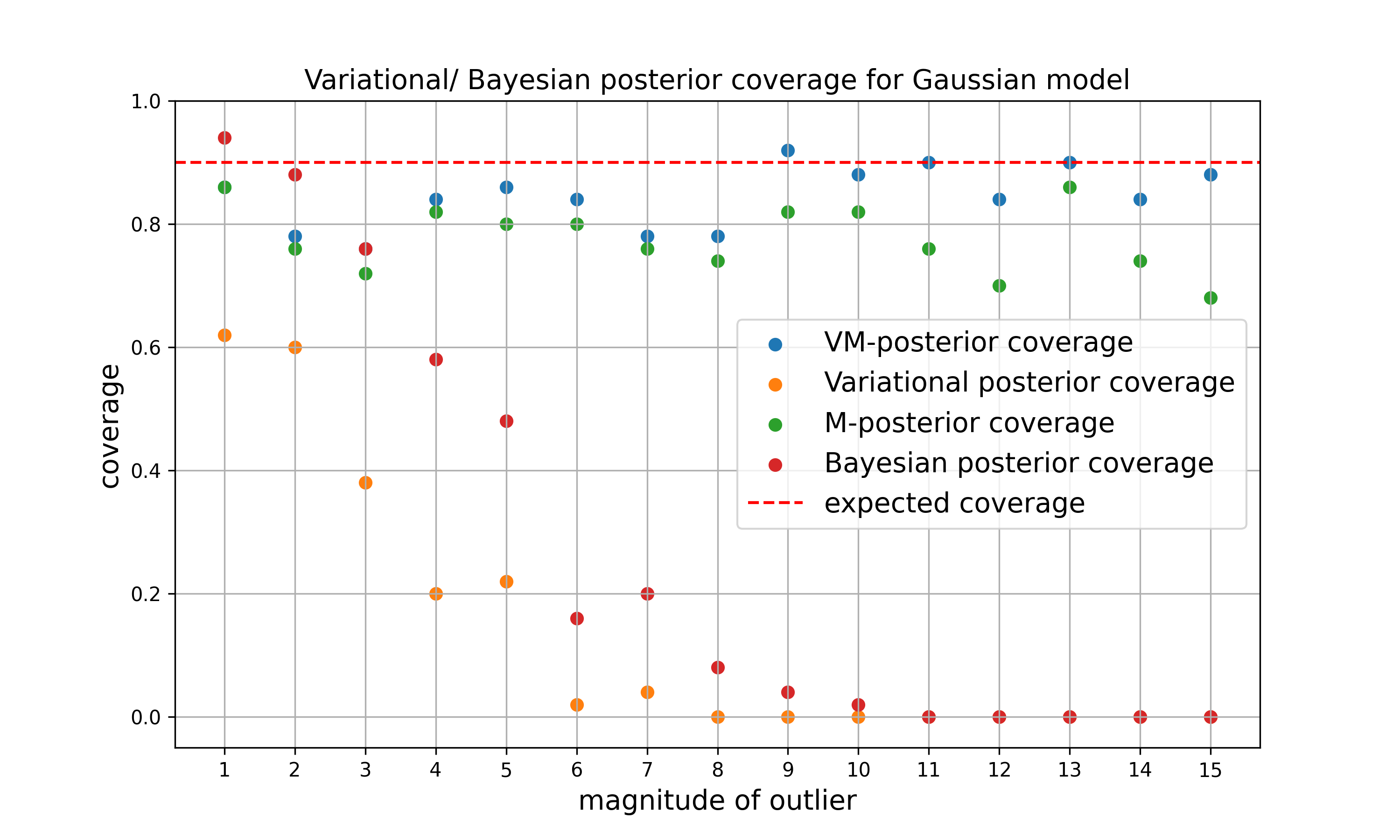}
     \caption{$\alpha$=0.1}
     \label{fig:90cov}

 \end{subfigure}
 \hfill
 \begin{subfigure}[b]{0.3\textwidth}
     \centering
     \includegraphics[width=\textwidth]{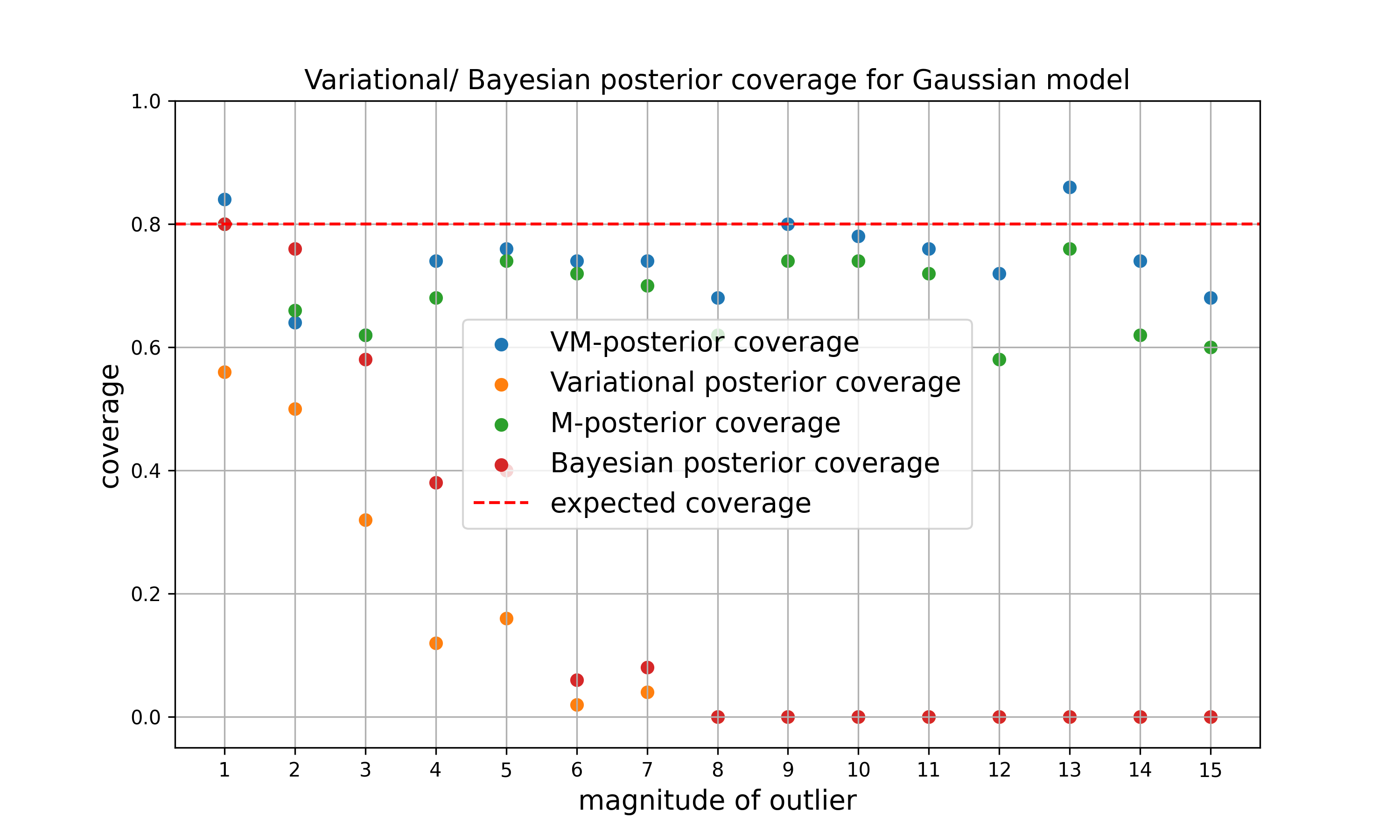}
     \caption{$\alpha$=0.2}
     \label{fig:80cov}
 \end{subfigure}
\caption{Posterior coverage for different levels of significance}
\label{fig:cov}
\end{figure}
In Figure \ref{fig:cov}, we plot the coverage of different credible interval levels—80\%, 90\%, and 95\%. The magnitude of the outlier increases as \( i \) times the maximum data value. The standard Bayes and VB methods exhibit low coverage when \( i = 2 \), and they almost completely fail to cover the true parameter as the outlier magnitude increases further. In contrast, both the M-posterior and the VM-posterior maintain coverage close to the expected levels, regardless of the outlier magnitude, across all three significance levels.
\begin{figure}[htbp!]
    \centering
    \includegraphics[width=0.5\linewidth]{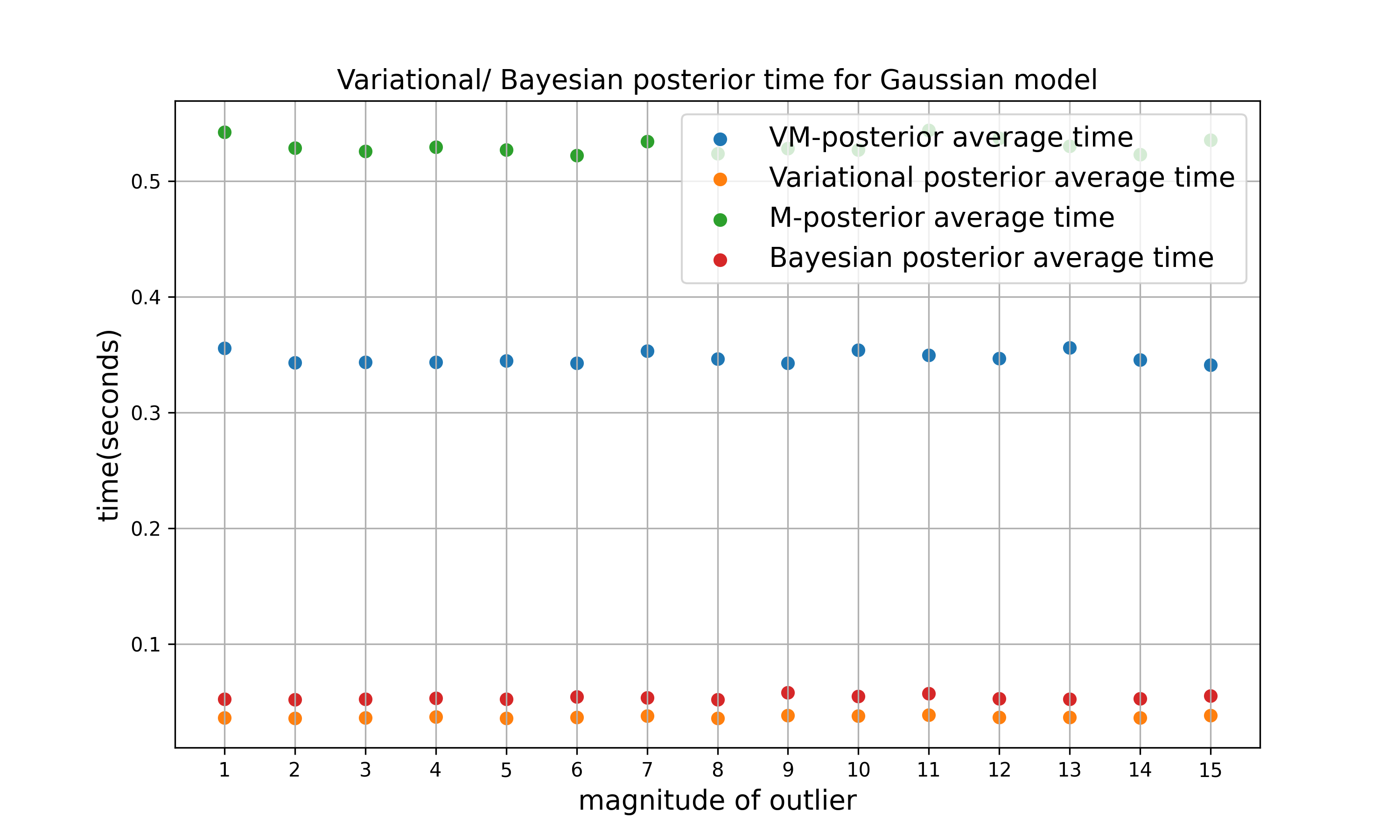}
    \caption{Posterior coverage computational cost}
    \label{fig:time}
\end{figure}
In Figure \ref{fig:time}, we demonstrate the computational advantages of the variational approach. The variational method is significantly faster than the standard Bayesian approach in our Gaussian examples. When applied to multiple data sets, the computational advantages of the variational approach become even more pronounced.

\subsection{Posterior Predictive Density for Gaussian Mixture}\label{sec-gmmsimu}
In this section, we calculate the posterior predictive distribution using a mixture of Student's t-distributions from \cite{Bishop2006Pattern-recogni00}:
\[
p(\widehat{\mathbf{x}} \mid \mathbf{X}) = \frac{1}{\widehat{\alpha}} \sum_{k=1}^K \alpha_k \operatorname{Student's- t}\left(\widehat{\mathbf{x}} \mid \mathbf{m}_k, \mathbf{L}_k, \nu_k+1-D\right)
\]
where the \( k^{\text{th}} \) component has mean \( \mathbf{m}_k \), and the precision matrix is given by
\[
\mathbf{L}_k = \frac{\left(\nu_k + 1 - D\right) \beta_k}{\left(1 + \beta_k\right)} \mathbf{W}_k
\]
In our case, when the data set size \( N \) is large, the predictive distribution approximates a mixture of Gaussians.

Following the approach in the previous section, we simulate 25 data sets, each containing 200 observations. Each data set \( \bold{x}_i = (x_{i,1}, \dots, x_{i,200}) \) consists of 199 independent observations from the Gaussian mixture distribution:
\[
x_{i,j} \sim 0.5 * N(\mu = 2, 1) + 0.5 * N(\mu = 4, 1) \quad \text{for } i = 1, \dots, 25 \text{ and } j = 1, \dots, 199
\]
The last entry in each data set, \( x_{i,200} \), was an outlier with its value increasing linearly for \( i = 1, \dots, 25 \) as \( x_{i,200} = \max (|\bold{x}_{i,1:199}|) \). The outlier index was unknown to the estimation algorithm.

We compute the variational posterior predictive distribution using both the standard variational inference and the VM-posterior procedure. The standard variational posterior predictive distribution is calculated following \cite{kapourani2019variational}. The VM-posterior is based on Algorithm \ref{alg2} to obtain the geometric median. We evaluate the robustness of the two methods as we increase the magnitude of the outlier.

    \begin{figure}[htbp!]
     \centering

     \begin{subfigure}[b]{0.45\textwidth}
         \centering
         \includegraphics[width=\textwidth]{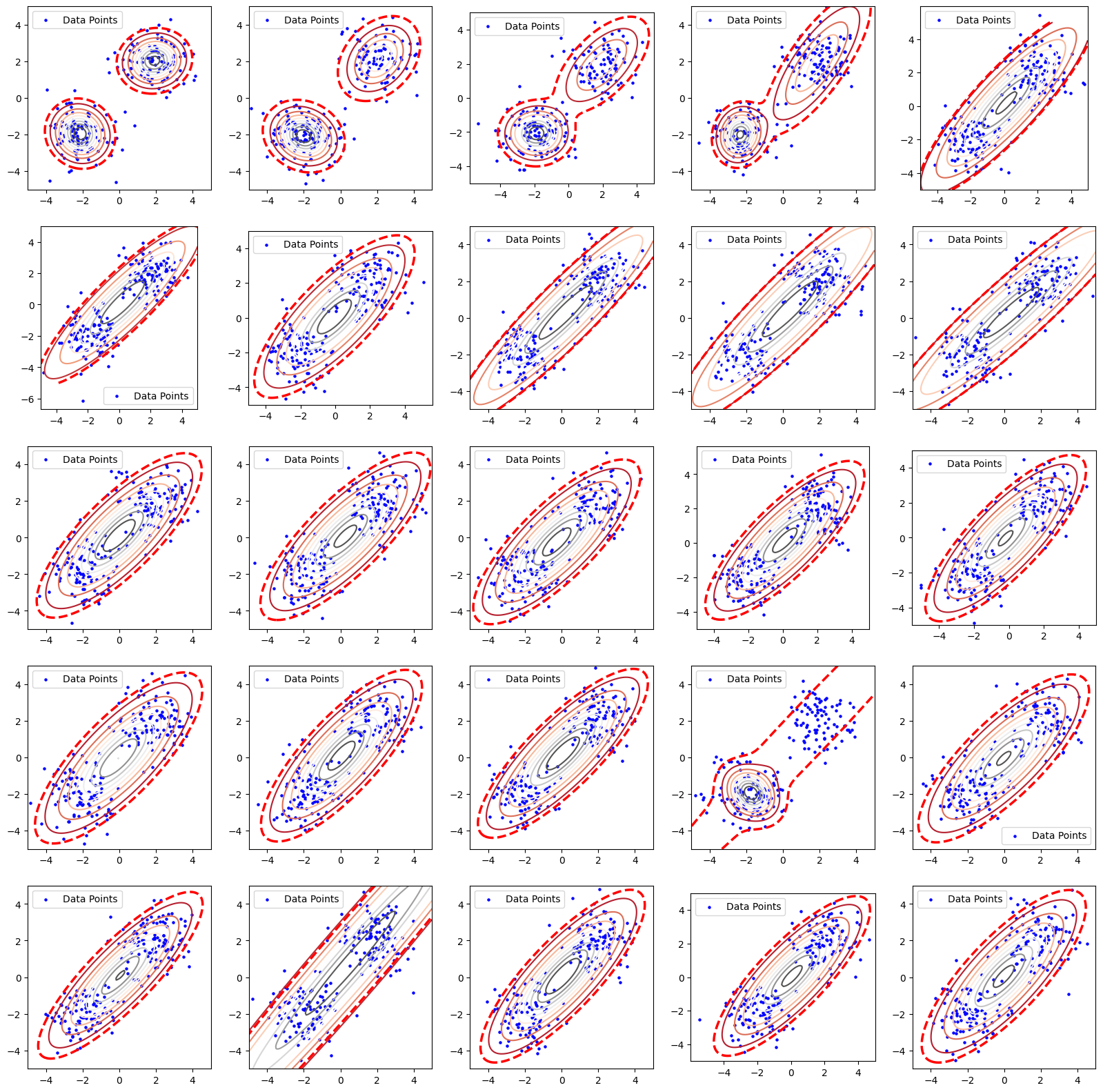}
         \caption{Standard Variational Inference}
         \label{fig:}

     \end{subfigure}
     \hfill
     \begin{subfigure}[b]{0.45\textwidth}
         \centering
         \includegraphics[width=\textwidth]{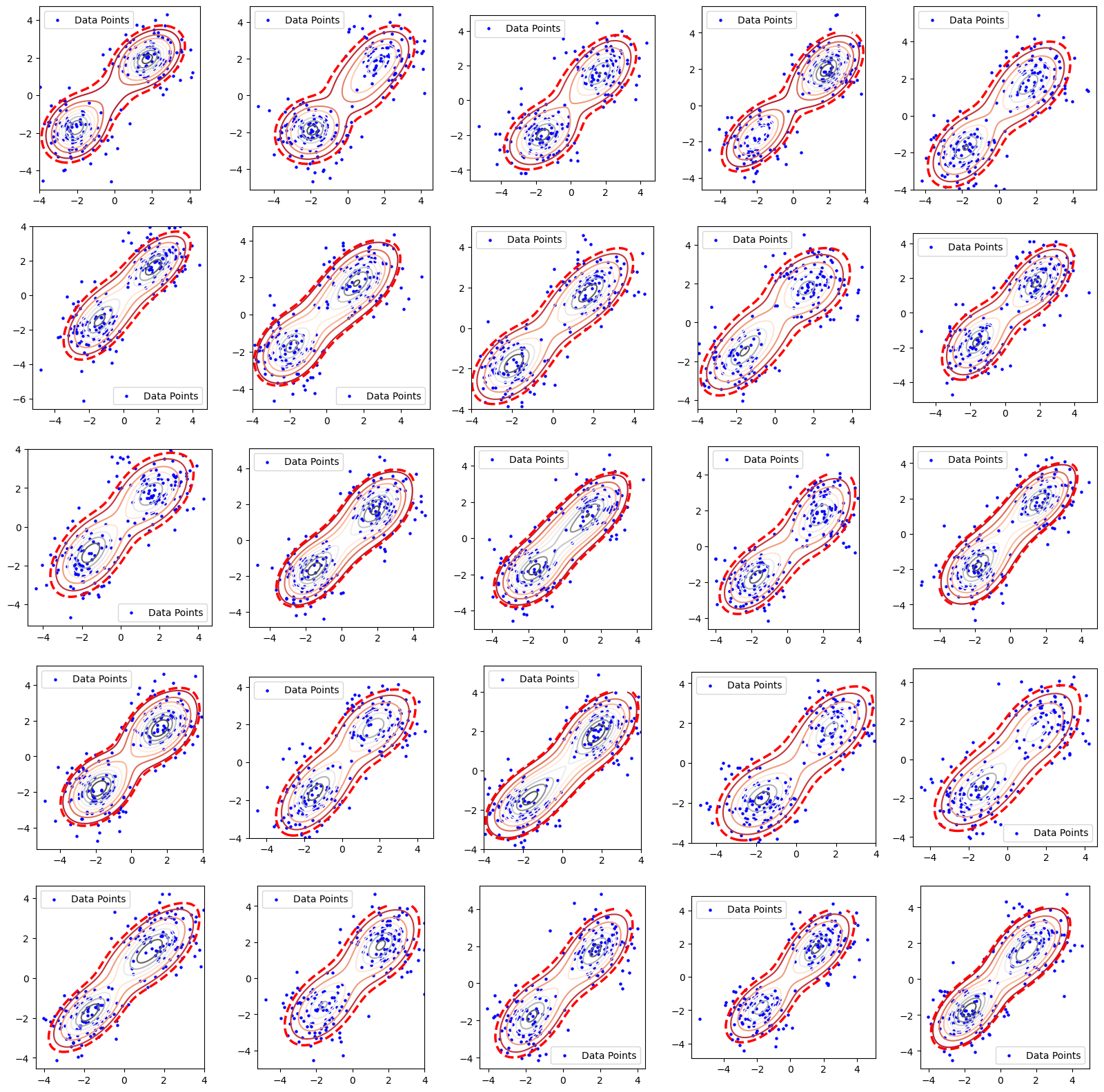}
         \caption{Median Variational Inference}
         \label{}
     \end{subfigure}
    \caption{Variational inference for Guassian mixture with the increasing magnitude of the outlier}
    \label{fig:}
    \end{figure}
We plot the 95\% posterior predictive credible region as the magnitude of the outlier increases. The outlier magnitude grows from left to right and from top to bottom. For better visualization, we omit plotting the outliers themselves. We observe that as the outlier magnitude increases, the standard method loses its ability to capture the mixture properties of the distribution. In contrast, the median method remains consistent regardless of the outlier magnitude.
\begin{figure}[htbp!]
 \centering
 \begin{subfigure}[b]{0.3\textwidth}
     \centering
     \includegraphics[width=\textwidth]{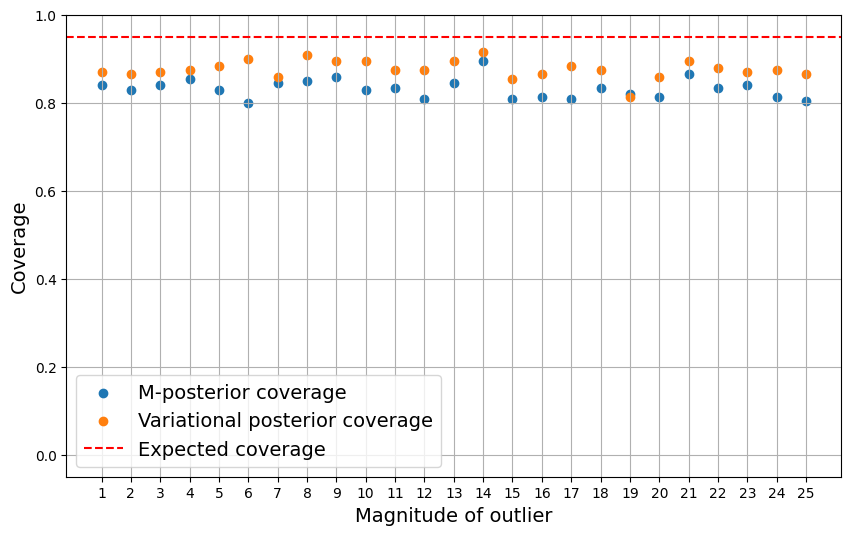}
     \caption{$\alpha$=0.05}
     \label{fig:gmm_95cov}

 \end{subfigure}
 \hfill
 \begin{subfigure}[b]{0.3\textwidth}
     \centering
     \includegraphics[width=\textwidth]{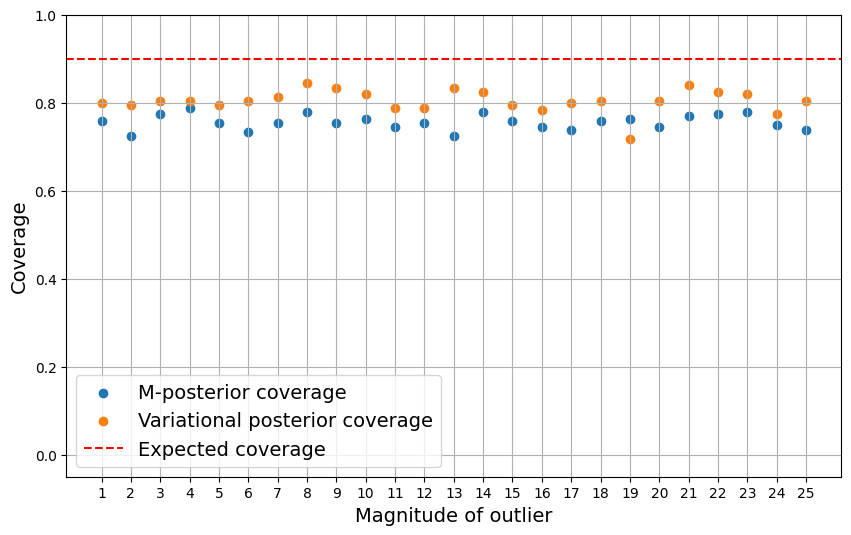}
     \caption{$\alpha$=0.1}
     \label{fig:gmm_90cov}

 \end{subfigure}
 \hfill
 \begin{subfigure}[b]{0.3\textwidth}
     \centering
     \includegraphics[width=\textwidth]{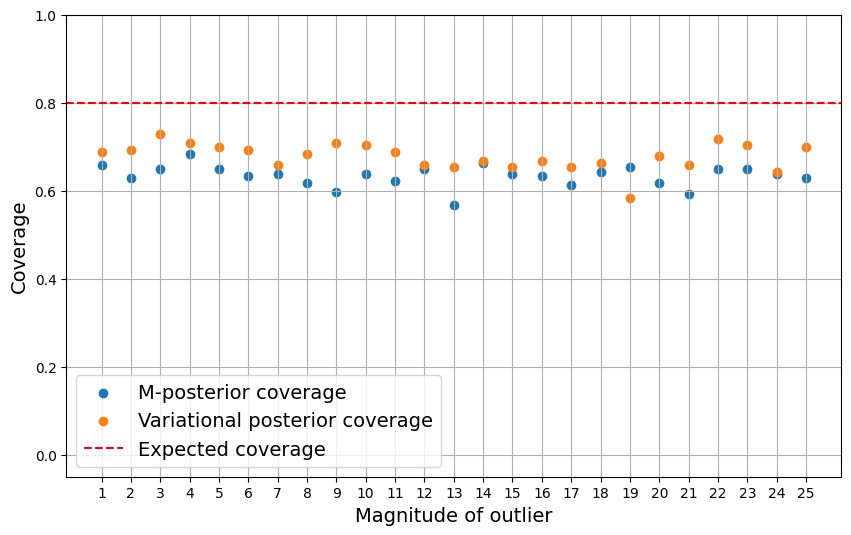}
     \caption{$\alpha$=0.2}
     \label{fig:gmm_80cov}
 \end{subfigure}
\caption{Posterior predictive coverage for different levels of significance}
\label{fig:gmm_cov}
\end{figure}
We also compared the posterior predictive coverage of the two methods in Figure \ref{fig:gmm_cov}. We counted the number of data points that lie within the predictive interval for each method. Although the median method achieves similar coverage to the standard method, the predictive region of the VM-Posterior has a smaller area and retains the mixture shape more accurately.
\subsection{Language Modeling: Latent Dirichlet Allocation (LDA)}\label{sec-lda}
Latent Dirichlet Allocation (LDA) is a generative model that uncovers hidden topics in a collection of documents based on the words within them.

Assume there are \( M \) documents, each containing \( N_m \) words, drawn from a vocabulary of size \( V \). LDA assumes there are \( K \) latent topics, and each document is a mixture of these topics. For each document \( m \), we represent the topic mixture with a vector \( \theta_m \), which has \( K \) elements (one for each topic). For each topic \( k \), we represent the distribution over words with a vector \( \phi_k \), which has \( V \) elements (one for each word in the vocabulary).

The generative process is as follows:
\begin{align*}
    \theta_m &\sim p_{\theta}, \quad \text{for each document } m = 1, 2, \dots, M, \\
    \phi_k &\sim p_{\phi}, \quad \text{for each topic } k = 1, 2, \dots, K, \\
    z_{m,j} &\sim \text{Mult}(\theta_m), \quad \text{for each word } j = 1, 2, \dots, N_m \text{ in document } m, \\
    w_{m,j} &\sim \text{Mult}(\phi_{z_{m,j}}), \quad \text{for each word } j = 1, 2, \dots, N_m \text{ in document } m.
\end{align*}

Here
\( \theta_m \) is the topic distribution for document \( m \),
\( \phi_k \) is the word distribution for topic \( k \),
\( z_{m,j} \) is the topic assigned to the \( j \)-th word in document \( m \),
\( w_{m,j} \) is the actual word at position \( j \) in document \( m \), based on its assigned topic \( z_{m,j} \).

In simpler terms, the first two equations assign "priors" to the document-topic distributions \( \theta_m \) and the topic-word distributions \( \phi_k \). The next two equations describe the process of selecting a topic for each word in a document and then choosing the word based on that topic.

In our simulation, we generate data from the LDA model with \( M = 19 \) documents, where each document contains \( N_m \sim \text{Pois}(10) \) words. The words are drawn from a vocabulary of size \( V = 4 \), and there are \( K = 2 \) topics. Additionally, we include one "outlier" document with \( N \) words, using the same vocabulary. Our primary interest is in the global parameter \( \phi_k \).

We generate data using \( p_\phi \), a Dirichlet distribution with parameter \( \beta = 1 \), and \( p_\theta \), a Dirichlet distribution with parameter \( \alpha = 2 \). Thus, we have the true value of \( \phi_k \), denoted as \( \phi_0 \).

To measure the performance of our simulation, we compute the KL divergence between \( \Pi(\phi_k \mid \text{Documents}) \) (the posterior distribution) and the true value \( \phi_0 \). We use HMC for the Bayesian posterior, the MFVB for the variational posterior. For the M-posterior and VM-posterior, we utilize Algorithm \ref{alg3}.
\begin{figure}[htbp!]
     \centering     \begin{subfigure}[b]{0.45\textwidth}
         \centering
         \includegraphics[width=\textwidth]{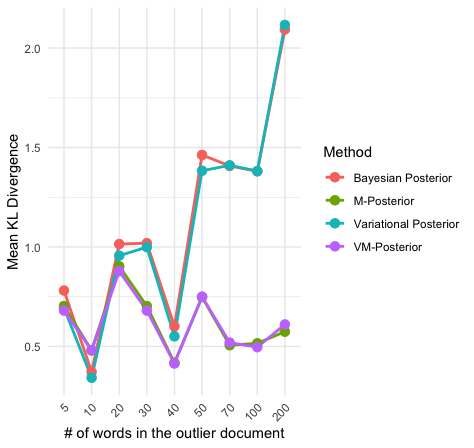}
         \caption{Mean KL Divergence Comparison}
         \label{fig:lda_kl}

     \end{subfigure}
     \hfill
     \begin{subfigure}[b]{0.45\textwidth}
         \centering
         \includegraphics[width=\textwidth]{plot/lda_time.png}
         \caption{Computational Time Comparison}
         \label{fig:lda_time}
     \end{subfigure}
    \caption{Variational inference for LDA model}
    \label{fig:lda}
    \end{figure}
\\
In Figure \ref{fig:lda}, we fit the LDA model to our simulated data using the standard Bayesian method, M-Posterior, Variational Bayes, and VM-Posterior. We increase the number of words in the outlier document and calculate the mean KL divergence of \( \phi \). The results show that our VM-Posterior approach performs as well as the Bayesian median approach. Furthermore, our VM-Posterior based approach is significantly faster than the Bayesian approach while achieving similar performance.

\subsection{Real Data Analysis}\label{sec-realdata}
In this section, we apply our methods to the "penguins" data set from Kaggle. This data set contains over 300 observations for three penguin species, with 6 covariates. We use the first 299 observations from the dataset and model it as a Gaussian mixture with 2 components for the covariates \texttt{culmen\_length} and \texttt{culmen\_depth}. We add an outlier with a value 5 times the largest in the dataset.

We fit both the standard variational posterior and the VM-Posterior to the penguins dataset. In Figure \ref{fig:penguins}, we plot the dataset without the added outlier, with each species of penguin represented by a different color. When we fit a Gaussian mixture with 2 clusters using the standard variational approach (Figure \ref{fig:penguins_vb}), we observe that the 95\% credible interval lacks meaningful information about the mixture structure, treating the two species as a single cluster. In contrast, the VM-Posterior (Figure \ref{fig:penguins_vm}) models the mixture data much better.
 \begin{figure}[htbp!]
     \centering

     \begin{subfigure}[b]{0.45\textwidth}
         \centering
         \includegraphics[width=\textwidth]{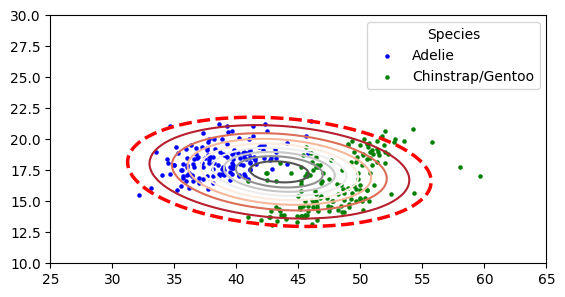}
         \caption{Variational posterior}
         \label{fig:penguins_vb}

     \end{subfigure}
     \hfill
     \begin{subfigure}[b]{0.45\textwidth}
         \centering
         \includegraphics[width=\textwidth]{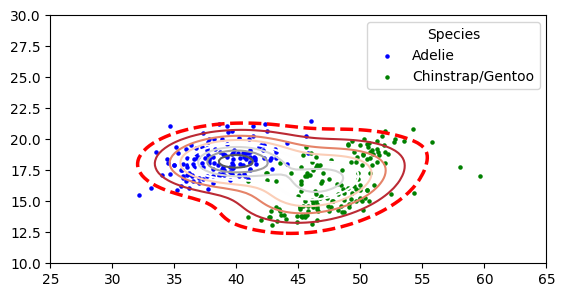}
         \caption{VM-Posterior}
         \label{fig:penguins_vm}
     \end{subfigure}
    \caption{Variational inference for penguins dataset with artificial outlier}
    \label{fig:penguins}
    \end{figure}

\section{Conclusion and Future Direction}\label{sec-conclusion} 
This paper introduced the VM-Posterior, a novel approach for robust variational inference that combines computational efficiency with resistance to outliers and data contamination. The key innovation lies in using geometric or metric median aggregation in the Wasserstein distance framework during the variational posterior aggregation step. This approach not only enhances robustness but also improves concentration properties, as weak concentrations of subset-based variational posteriors are transformed into strong concentrations under the Wasserstein geometric or metric median. By leveraging this aggregation method, the VM-Posterior achieves significant improvements over traditional variational methods while maintaining comparable robustness to established techniques like the M-Posterior.

Through rigorous theoretical development and diverse experiments, we demonstrated the VM-Posterior’s robustness, efficiency, and adaptability. Theoretical results established the VM-Posterior’s strong concentration properties, asymptotic normality, and finite-sample guarantees, providing a strong foundation for its practical utility. The algorithmic framework, designed for mean-field and Gaussian variational families, was extended to handle general distributions using efficient numerical methods. Empirical results across synthetic and real-world datasets confirmed the VM-Posterior’s superior performance in retaining credible coverage, preserving mixture structures, and resisting contamination effects, all while achieving computational efficiency that is crucial for large-scale data settings.

Building on these results, the framework showcased its versatility across multivariate Gaussian models, Gaussian mixtures, and the Latent Dirichlet Allocation (LDA) model. Real-world applicability was further validated using the penguins dataset, where the VM-Posterior demonstrated its ability to model complex data distributions while mitigating the effects of outliers. By balancing the strengths of variational inference and robust Bayesian estimation, the VM-Posterior positions itself as a powerful and practical tool for modern Bayesian inference challenges.

Future research may focus on extending the theoretical foundation of the VM-Posterior to justify its empirical success in Gaussian mixture models. While the robustness results developed here apply to general variational families, extending Corollary 7 in \cite{wang2019frequentist} to handle mixtures of Gaussians would establish strong asymptotic normality results for this specific family. Such advancements would solidify the numerical findings presented for Gaussian mixtures and provide a more comprehensive theoretical framework.

Another promising direction involves broadening the applicability of the VM-Posterior to more complex models and exploring its use in diverse machine learning and statistical settings. Applications in Bayesian deep learning, hierarchical models, or dynamic systems could showcase its robustness and computational efficiency in handling the challenges posed by high-dimensional, non-linear, and evolving data structures.






\newpage
\appendix
\section{Proof of Theorem \ref{thm-rob-gm} and Theorem \ref{thm-rob-mm}} 
Theorem \ref{thm-rob-gm} and Theorem \ref{thm-rob-mm} are an adaptation of Theorem 3.1 in \cite{minsker2015geometric}. Precisely Theorem \ref{thm-rob-gm} and Theorem \ref{thm-rob-mm} follows directly from the following Theorem in \cite{minsker2017robust}.

\begin{theorem} 
\label{ProfLinPaper}
{\bf{a}}.
Assume that $(\mathbb{H},\|\cdot\|)$ is a Hilbert space and $\theta_0 \in \mathbb{H}$. Let $\hat{\theta}_1, \ldots, \hat{\theta}_m \in \mathbb{H}$ be a collection of independent random variables. Let the constants $\alpha, q, \gamma$ be such that $0<q<\alpha<1 / 2$, and $0 \leq \gamma<\frac{\alpha-q}{1-q}$. Suppose $\varepsilon>0$ is such that for all $j, 1 \leq j \leq\lfloor(1-\gamma) m\rfloor+1$,
$$
\operatorname{Pr}\left(\left\|\hat{\theta}_j-\theta_0\right\|>\varepsilon\right) \leq q .
$$
Let $\hat{\theta}_*=\operatorname{med}_g\left(\hat{\theta}_1, \ldots, \hat{\theta}_m\right)$ be the geometric median of $\left\{\hat{\theta}_1, \ldots, \hat{\theta}_m\right\}$. Then
$$
\operatorname{Pr}\left(\left\|\hat{\theta}_*-\theta_0\right\|>C_\alpha \varepsilon\right) \leq\left[e^{(1-\gamma) \psi\left(\frac{\alpha-\gamma}{1-\gamma}, q\right)}\right]^{-m}
$$
where $C_\alpha=(1-\alpha) \sqrt{\frac{1}{1-2 \alpha}}$.
\\
{\bf{b}}.  Assume that $(\mathbb{Y}, d)$ is a metric space and $\theta_0 \in \mathbb{Y}$. Let $\hat{\theta}_1, \ldots, \hat{\theta}_m \in \mathbb{Y}$ be a collection of independent random variables. Let the constants $q, \gamma$ be such that $0<q<\frac{1}{2}$ and $0 \leq \gamma<\frac{1 / 2-q}{1-q}$. Suppose $\varepsilon>0$ are such that for all $j, 1 \leq j \leq\lfloor(1-\gamma) m\rfloor+1$,
$$
\operatorname{Pr}\left(d\left(\hat{\theta}_j, \theta_0\right)>\varepsilon\right) \leq q .
$$
Let $\hat{\theta}_*=\operatorname{med}_0\left(\hat{\theta}_1, \ldots, \hat{\theta}_m\right)$. Then
$$
\operatorname{Pr}\left(d\left(\hat{\theta}_*, \theta_0\right)>3 \varepsilon\right) \leq e^{-m(1-\gamma) \psi\left(\frac{1 / 2-\gamma}{1-\gamma}, q\right)}
$$
\end{theorem}
To get Theorem \ref{thm-rob-gm} in our paper, we take $q=\frac{1}{7}$ and $\alpha=\frac{3}{7}$ in part {\bf{(a)}} of the previously theorem. Moreover, Theorem \ref{thm-rob-mm} is followed by part {\bf{(b)}}, when considering $q=\frac{1}{4}$.

For completeness, we present the proof of the above Theorem. To this end, we make use of the following lemma (see lemma 2.1 in \cite{minsker2015geometric}).
\begin{lemma}
\label{lemma21}
 Let $\mathbb{H}$ be a Hilbert space, $x_1, \ldots, x_m \in \mathbb{H}$ and let $x_*$ be their geometric median. Fix $\alpha \in\left(0, \frac{1}{2}\right)$ and assume that $z \in \mathbb{H}$ is such that $\left\|x_*-z\right\|>C_\alpha r$, where
$$
C_\alpha=(1-\alpha) \sqrt{\frac{1}{1-2 \alpha}}
$$
and $r>0$. Then there exists a subset $J \subseteq\{1, \ldots, m\}$ of cardinality $|J|>\alpha m$ such that for all $j \in J,\left\|x_j-z\right\|>r$.
\end{lemma}
We now provide the proof of Theorem \ref{ProfLinPaper}.
\begin{proof}
Assume that event $\mathcal{E}:=\left\{\left\|\hat{\theta}_*-\theta_0\right\|>C_\alpha \varepsilon\right\}$ occurs. Lemma \ref{lemma21} implies that there exists a subset $J \subseteq\{1, \ldots, m\}$ of cardinality $|J| \geq \alpha k$ such that $\left\|\hat{\theta}_j-\theta_0\right\|>\varepsilon$ for all $j \in J$, hence
\begin{align}
\operatorname{Pr}(\mathcal{E}) \leq & \operatorname{Pr}\left(\sum_{j=1}^m I\left\{\left\|\hat{\theta}_j-\theta_0\right\|>\varepsilon\right\}>\alpha m\right) \leq \\
& \operatorname{Pr}\left(\sum_{j=1}^{\lfloor(1-\gamma) m\rfloor+1} I\left\{\left\|\hat{\theta}_j-\theta_0\right\|>\varepsilon\right\}>(\alpha-\gamma) m \frac{\lfloor(1-\gamma) m\rfloor+1}{\lfloor(1-\gamma) m\rfloor+1}\right) \leq \\
& \operatorname{Pr}\left(\sum_{j=1}^{\lfloor(1-\gamma) m\rfloor+1} I\left\{\left\|\hat{\theta}_j-\theta_0\right\|>\varepsilon\right\}>\frac{\alpha-\gamma}{1-\gamma}(\lfloor(1-\gamma) m\rfloor+1)\right) .
\end{align}
If $W$ has Binomial distribution $W \sim B(\lfloor(1-\gamma) m\rfloor+1, q)$, then
\begin{align}
\operatorname{Pr}\left(\sum_{j=1}^{\lfloor(1-\gamma) m\rfloor+1} I\left\{\left\|\hat{\theta}_j-\theta_0\right\|>\varepsilon\right\}\right. & \left.>\frac{\alpha-\gamma}{1-\gamma}(\lfloor(1-\gamma) m\rfloor+1)\right) \leq \\
& \operatorname{Pr}\left(W>\frac{\alpha-\gamma}{1-\gamma}(\lfloor(1-\gamma) m\rfloor+1)\right)
\end{align}
(see Lemma 23 in \cite{lerasle2011robust} for a rigorous proof of this fact). Chernoff bound (e.g., Proposition A.6.1 in \cite{Vaart1996Weak-convergenc00}), together with an obvious bound $\lfloor(1-\gamma) m\rfloor+1>(1-\gamma) m$, implies that
$$
\operatorname{Pr}\left(W>\frac{\alpha-\gamma}{1-\gamma}(\lfloor(1-\gamma) m\rfloor+1)\right) \leq \exp \left(-m(1-\gamma) \psi\left(\frac{\alpha-\gamma}{1-\gamma}, q\right)\right) \text {. }
$$
To establish part $\mathbf{b}$, we proceed as follows: let $\mathcal{E}_1$ be the event
$\mathcal{E}_1=\left\{\right.$ more than a half of events $d\left(\hat{\theta}_j, \theta_0\right) \leq \varepsilon, j=1 \ldots m$ occur $\}$.
Assume that $\mathcal{E}_1$ occurs. Then we clearly have $\varepsilon_* \leq \varepsilon$, where $\varepsilon_*$ is defined as
$$\begin{gathered}\varepsilon_*:=\inf \{\varepsilon>0: \exists j=j(\varepsilon) \in\{1, \ldots, m\} \text { and } I(j) \subset\{1, \ldots, m\} \text { such that } \\ \left.|I(j)|>\frac{m}{2} \text { and } \forall i \in I(j), d\left(\hat{\theta}_i, \hat{\theta}_j\right) \leq 2 \varepsilon\right\}.
\end{gathered}$$
Indeed, for any $\theta_{j_1}, \theta_{j_2}$ such that $d\left(\hat{\theta}_{j_i}, \theta_0\right) \leq \varepsilon, i=1,2$, triangle inequality gives $d\left(\theta_{j_1}, \theta_{j_2}\right) \leq 2 \varepsilon$. By the definition of $\hat{\theta}_*$, inequality $d\left(\hat{\theta}_*, \hat{\theta}_j\right) \leq 2 \varepsilon_* \leq 2 \varepsilon$ holds for at least a half of $\left\{\hat{\theta}_1, \ldots, \hat{\theta}_m\right\}$, hence, it holds for some $\hat{\theta}_{\tilde{j}}$ with $d\left(\hat{\theta}_{\tilde{j}}, \theta_0\right) \leq \varepsilon$. In turn, this implies (by triangle inequality) $d\left(\hat{\theta}_*, \theta_0\right) \leq 3 \varepsilon$. We conclude that
$$
\operatorname{Pr}\left(d\left(\hat{\theta}_*, \theta_0\right)>3 \varepsilon\right) \leq \operatorname{Pr}\left(\mathcal{E}_1\right)
$$
The rest of the proof repeats the argument of part {\bf{a}} since
$$
\operatorname{Pr}\left(\mathcal{E}_1^c\right)=\operatorname{Pr}\left(\sum_{j=1}^m I\left\{d\left(\hat{\theta}_j, \theta_0\right)>\varepsilon\right\} \geq \frac{m}{2}\right),
$$
where $\mathcal{E}_1^c$ is the complement of $\mathcal{E}_1$.
\end{proof}

\section{Proof of Theorem \ref{mainT-1}}
\begin{proof}
To start, we consider the events $$\mathbb{B}_l=\mathbb{B}_l\left(ml\varepsilon_l^2, Q_{n,m}^{(j)}, \Pi \right):=\Big\{\int \Big(\prod_{j\in G_j}\frac{p_\theta}{p_0}(X_j)\Big)^md\Pi(\theta)\ge\exp(-ml\varepsilon_l^2-\mathrm{KL}(Q_{n,m}^{(j)},\Pi))\Big\},$$
and
$$\mathbb{A}_l=\Big\{\int_{\rho(\theta,\theta_0)>R\varepsilon_l}\Big(\prod_{j\in G_j}\frac{p_\theta}{p_0}(X_j)\Big)^m d\Pi(\theta)\le e^{-\mathfrak{c}_1^lR^2ml\varepsilon_l^2}\Big\}.$$ 
Next, observe that
\begin{align}
\label{aux1-1}
&\widehat{Q}_{n,m}^{(j)}(\rho(\theta,\theta_0)\geq R\varepsilon_l)\nonumber
\\
=&\widehat{Q}_{n,m}^{(j)}(\rho(\theta,\theta_0)\geq R\varepsilon_l)\mathbf{1}_{\mathbb{B}_l}+\widehat{Q}_{n,m}^{(j)}(\rho(\theta,\theta_0)\geq R\varepsilon_l)\mathbf{1}_{\mathbb{B}_l^c}\nonumber
\\
\le&\widehat{Q}_{n,m}^{(j)}(\rho(\theta,\theta_0)\geq R\varepsilon_l)\mathbf{1}_{\mathbb{B}_l}+\mathbf{1}_{\mathbb{B}_l^c},
\end{align}
where the fact that $\widehat{Q}_{n,m}^{(j)}(\rho(\theta,\theta_0)\geq R\varepsilon_l)\le 1$ has been used.
Similarly, we have that 
\begin{align}
\label{aux2-1}
&\widehat{Q}_{n,m}^{(j)}(\rho(\theta,\theta_0)\geq R\varepsilon_l)\mathbf{1}_{B_l}\nonumber
\\
=& \widehat{Q}_{n,m}^{(j)}(\rho(\theta,\theta_0)\geq R\varepsilon_l)\mathbf{1}_{\mathbb{B}_l\cap \mathbb{A}_l}+\widehat{Q}_{n,m}^{(j)}(\rho(\theta,\theta_0)\geq R\varepsilon_l)\mathbf{1}_{\mathbb{B}_l\cap \mathbb{A}_l^c}\nonumber
\\
\le&\widehat{Q}_{n,m}^{(j)}(\rho(\theta,\theta_0)\geq R\varepsilon_l)\mathbf{1}_{\mathbb{B}_l\cap \mathbb{A}_l}+\mathbf{1}_{\mathbb{A}_l^c}.
\end{align}
Therefore, from Equation (\ref{aux1-1}) and Equation (\ref{aux2-1}),
\begin{align}
\label{aux3-1}
\widehat{Q}_{n,m}^{(l)}(\rho(\theta,\theta_0)\geq R\varepsilon_l)\le& \widehat{Q}_{n,m}^{(l)}(\rho(\theta,\theta_0)\geq R\varepsilon_l)\mathbf{1}_{\mathbb{B}_l\cap \mathbb{A}_l}+\mathbf{1}_{\mathbb{A}_l^c}+\mathbf{1}_{\mathbb{B}_l^c}
\end{align}
Now, we observe that  using  Lemma \ref{lemma1} with $v_l=ml\varepsilon_l^2$, we get that
\begin{align}
\label{equation3-1}
&\widehat{Q}_{n,m}^{(j)}(\rho(\theta,\theta_0)\geq R\varepsilon_l)\mathbf{1}_{\mathbb{B}_l\cap \mathbb{A}_l}\nonumber
\\
\le& \frac{1}{v_l}\Big[\mathrm{KL}(\widehat{Q}_{n,m}^{(j)},\Pi_{n,m}^{\vert G_j\vert})\mathbf{1}_{\mathbb{B}_l\cap \mathbb{A}_l}\Big]+\exp(v_l)\Big[\Pi_{n,m}(\rho(\theta,\theta_0)\geq R\varepsilon_l|G_j)\mathbf{1}_{\mathbb{B}_l\cap \mathbb{A}_l}\Big].
\end{align}
By Equation (\ref{aux3-1}) and Equation (\ref{equation3-1}) it follows that,
\begin{align*}
   \mathbb{E}_0\Big(\widehat{Q}_{n,m}^{(j)}(\rho(\theta,\theta_0)\geq R\eps_l)\Big)
    &\le \mathbb{E}_0\left(\Pi_{n,m}(\rho(\theta,\theta_0)\geq R\varepsilon_l|G_j)\mathbf{1}_{\mathbb{B}_l\cap \mathbb{A}_l}\right) 
    \\
&+\mathbb{E}_0(\mathbf{1}_{\mathbb{A}_l^c})+\mathbb{E}_0(\mathbf{1}_{\mathbb{B}_l^c})+\mathbb{E}_0(\frac{1}{v_l}\Big[\mathrm{KL}(\widehat{Q}_{n,m}^{(j)},\Pi_{n,m}^{\vert G_j\vert})\Big])
\end{align*}
Then by Markov's inequality,
\begin{align}
\label{equation1-1}  
  & \Pr\Big(\widehat{Q}_{n,m}^{(j)}(\rho(\theta,\theta_0)\geq R\eps_l)\Big)\nonumber
    \\
    &\le \exp(v_l)\mathbb{E}_0(\Pi_{n,m}(\rho(\theta,\theta_0)\geq R\varepsilon_l|G_j)\mathbf{1}_{\mathbb{B}_l\cap \mathbb{A}_l}) +P({\mathbb{A}_l^c})+P({\mathbb{B}_l^c})+\frac{\mathbb{E}_0\Big[\mathrm{KL}(\widehat{Q}_{n,m}^{(j)},\Pi_{n,m}^{\vert G_j\vert})\Big]}{v_l}\nonumber
    \\
    &=I_1+I_2+I_3+I_4.
\end{align}
Now, we bound each of the terms $I_1,$ $I_2$, $I_3$ and $I_4.$ For the term $I_1$, we notice that,
on the event $\mathbb{A}_l\cap \mathbb{B}_l$ it is satisfied 
$$\int_{\rho(\theta,\theta_0)>R\varepsilon_l}\Big(\prod_{j\in G_j}\frac{p_\theta}{p_0}(X_j)\Big)^m d\Pi(\theta)\le \exp(-\mathfrak{c}_1^lR^2ml\varepsilon_l^2)$$
and,
\begin{align*}
    \int\Big(\prod_{j\in G_j}\frac{p_\theta}{p_0}(X_j)\Big)^m d\Pi(\theta)
    \ge& \exp(-ml\varepsilon_l^2-\mathrm{KL}(Q_{n,m}^{(j)^*},\Pi)).
\end{align*}
In consequences,
\begin{align}
\label{equation2-1}
    \Pi_{n,m}(\rho(\theta,\theta_0)\geq R\varepsilon_l|G_j)\mathbf{1}_{\mathbb{B}_l\cap \mathbb{A}_l}\le \exp( ml\varepsilon_l^2+\mathrm{KL}(Q_{n,m}^{(j)^*},\Pi)-\mathfrak{c}_1^lR^2ml\varepsilon_l^2)
\end{align}
Therefore, by Equation (\ref{equation2-1}), and by Assumption \ref{Prior-cond}, we have that
\begin{align}
    \label{equation4-1}
    I_1
\le&\exp(3ml\varepsilon_l^2+\mathfrak{c}_5^l l\varepsilon_{l}^2-\mathfrak{c}_1^lR^2ml\varepsilon_l^2).
\end{align}
In the following line of arguments we proceed to bound $I_2$ and $I_3.$ For $I_2$, Since if Theorem \ref{th:wong} holds with $\zeta:=\varepsilon_l$, then it also holds with $\zeta:=L\varepsilon_l$ for any $L\ge1,$ we have that by Assumption \ref{test-cond}
\begin{align*}
    P_0^{(l)}(\mathbb{A}_l)=&
    P_0^{(l)}\Big(\int_{\rho(\theta,\theta_0)>R\varepsilon}\Big(\prod_{j\in G_j}\frac{p_\theta}{p_0}(X_j)\Big)^m d\Pi(\theta)\le \exp(-\mathfrak{c}_1^lR^2ml\varepsilon_l^2)\Big)
    \\
    \ge& 
    1-4\exp(-(\mathfrak{c}_2^l)^2R^2l\varepsilon_l^2).
\end{align*}
This is,
\begin{align}
    \label{equation5-1}
    I_2\le4\exp(-(\mathfrak{c}_2^l)^2R^2l\varepsilon_l^2).
\end{align}
Finally, we bound the term $I_3.$ For this purpose, we apply Lemma \ref{lemma1} with \break $\mathrm{F}=\log \Big(\Big(\prod_{j\in G_j}\frac{p_\theta}{p_0}(X_j)\Big)^m\Big),$ $ Q_0=Q_{n,m}^{(j)^*}$ and $\Pi_0=\Pi$ to obtain
$$
\log \int \Big(\prod_{j\in G_j}\frac{p_\theta}{p_0}(X_j)\Big)^m \mathrm{d} \Pi(\boldsymbol{\theta}) \geq \int \log \left(\Big(\prod_{j\in G_j}\frac{p_\theta}{p_0}(X_j)\Big)^m\right) \mathrm{d} Q_{n,m}^{(j)^*}(\mathrm{~d} \boldsymbol{\theta})-\mathrm{KL}(Q_{n,m}^{(j)^*}, \Pi) .
$$
Using the definition of event $\mathbb{B}$, we have that
\begin{align*}
    \mathrm{P}_{0}^{(l)}\left(\mathbb{B}_l^C\right) & =\mathrm{P}_{0}^{(l)}\left(\log \int \Big(\prod_{j\in G_j}\frac{p_\theta}{p_0}(X_j)\Big)^m \mathrm{d} \Pi(\boldsymbol{\theta})<-ml\varepsilon_l^2-\mathrm{KL}(Q_{n,m}^{(j)^*}, \Pi)\right), 
\end{align*}
from where,
\begin{align*}
    \mathrm{P}_{0}^{(l)}\left(\mathbb{B}_l^C\right)\le&\mathrm{P}_{0}^{(l)}\left(-\int \log \left(\Big(\prod_{j\in G_j}\frac{p_0}{p_\theta}(X_j)\Big)^m\right) \mathrm{d} Q_{n,m}^{(j)^*}(\mathrm{~d} \boldsymbol{\theta}) \leq-ml\varepsilon_l^2\right). 
\end{align*}
Hence,
$$
\begin{aligned}
\mathrm{P}_{0}^{(l)}\left(\mathbb{B}_l^C\right) &  \leq \mathrm{P}_{0}^{(l)}\left(\int 0 \vee \log \Big(\prod_{j\in G_j}\frac{p_0}{p_\theta}(X_j)\Big) \mathrm{d} Q_{n,m}^{(j)^*}(\boldsymbol{\theta}) \geq l\varepsilon_l^2\right) \\
& \leq \frac{1}{l\varepsilon_l^2} \mathrm{P}_{0}^{(l)}\left[\int 0 \vee \log \Big(\prod_{j\in G_j}\frac{p_0}{p_\theta}(X_j)\Big) \mathrm{d} Q_{n,m}^{(j)^*}(\boldsymbol{\theta})\right] \\
& \leq \frac{1}{l\varepsilon_l^2} Q_{n,m}^{(j)^*}\left[\mathrm{KL}\left(\mathrm{P}_{0}^{(l)}, \mathrm{P}_{\theta}^{(l)}\right)+\sqrt{\frac{1}{2}\mathrm{KL}\left(\mathrm{P}_{0}^{(l)}, \mathrm{P}_{\theta}^{(l)}\right)}\right] \\
& \leq \frac{1}{l\varepsilon_l^2}\left(2 Q_{n,m}^{(j)^*}\left[\mathrm{KL}\left(\mathrm{P}_{0}^{(l)}, \mathrm{P}_{\theta}^{(l)}\right)\right]+1\right) .
\end{aligned}
$$
Here we use Markov's inequality for the fourth line and use Fubini's theorem and Lemma B.13 of \cite{ghosal2017fundamentals} for the fifth line. For the last line, we use a simple inequality $z+\sqrt{z / 2} \leq z+(z \vee 1) \leq 2 z+1$ for $z \geq 0$. 
Thus, by Assumption \ref{Prior-cond}
\begin{align}
    \label{equation6-1}
    I_3=\mathrm{P}_{0}^{(l)}\left(\mathbb{B}_l^C\right)\le \frac{1}{l\varepsilon_l^2}+2\mathfrak{c}_5^{l}\frac{1}{m}
\end{align}
Finally, we analyze the term $I_4.$ Using Proposition \ref{prop1}, we get that
\begin{align}
\label{equation8-1}I_4=&\frac{\mathbb{E}_0\Big[\mathrm{KL}(\widehat{Q}_{n,m}^{(j)},\Pi_{n,m}^{\vert G_j\vert})\Big]}{v_l}\le \frac{\mathfrak{c}_5^{l}l\varepsilon_l^2}{ml\varepsilon_l^2}=\frac{\mathfrak{c}_5^{l}}{m}.
\end{align}
To conclude, we observe that by Equation (\ref{equation1-1}), (\ref{equation4-1}), (\ref{equation5-1}), (\ref{equation6-1}), and (\ref{equation8-1}), we obtain that
\begin{align*}
    \mathbb{E}_0\Big(\widehat{Q}_{n,m}^{(j)}(\rho(\theta,\theta_0)\geq R\eps_l)\Big)\nonumber
    &\le\exp(3ml\varepsilon_l^2+\mathfrak{c}_5^l l\varepsilon_{l}^2-\mathfrak{c}_1^lR^2ml\varepsilon_l^2)
    \\
    &+4\exp(-(\mathfrak{c}_2^l)^2R^2l\varepsilon_l^2)+\frac{1}{l\varepsilon_l^2}+2\mathfrak{c}_5^{l}\frac{1}{m}+\frac{\mathfrak{c}_5^{l}}{m}.
\end{align*}
Furthermore, using the fact that $l\varepsilon_l^2\ge 1$ and considering $R$ such that 
$R>\sqrt{\frac{6+2\mathfrak{c}_5^l}{\mathfrak{c}_1^l}},$
we get that
\begin{align*}
    &\le\exp(-\frac{\mathfrak{c}_1^l}{2}R^2l\varepsilon_l^2)+4\exp(-(\mathfrak{c}_2^l)^2R^2l\varepsilon_l^2)+\frac{1}{l\varepsilon_l^2}+3\mathfrak{c}_5^{l}\frac{1}{m},
\end{align*}
and the claim is followed.
\end{proof}
\newpage

\section{Proof of Theorem \ref{mainT}}
The proof of Theorem \ref{mainT} follows a similar line of reasoning as the proof of Theorem \ref{mainT-1}, with minor modifications. For the sake of completeness, we provide the full details of the proof below, including steps that were previously analyzed in the proof of Theorem \ref{mainT-1}.
\begin{proof}
By the definition of Wasserstein distance, we observe that
\begin{align*}
d_{W_1,\rho}(\widehat{Q}_{n,m}^{(j)},\delta_0)=&\int_\Theta \rho(\theta,\theta_0) d \widehat{Q}_{n,m}^{(j)}(\theta) \\
\leq
&
R\eps_l+
\int\limits_{\rho(\theta,\theta_0)\geq R\eps_l}d\widehat{Q}_{n,m}^{(j)}(\theta) \\
=& R\eps_l+\widehat{Q}_{n,m}^{(j)}(\rho(\theta,\theta_0)\geq R\eps_l).
\end{align*}
To obtain the result it remains to bound
$\widehat{Q}_{n,m}^{(j)}(\rho(\theta,\theta_0)\geq R\eps_l)$.
To this end, we consider the events $$\mathbb{B}_l=\mathbb{B}_l\left(ml\varepsilon_l^2, Q_{n,m}^{(j)}, \Pi \right):=\Big\{\int \Big(\prod_{j\in G_j}\frac{p_\theta}{p_0}(X_j)\Big)^md\Pi(\theta)\ge\exp(-ml\varepsilon_l^2-\mathrm{KL}(Q_{n,m}^{(j)},\Pi))\Big\},$$
and
$$\mathbb{A}_l=\Big\{\int_{\rho(\theta,\theta_0)>R\varepsilon_l}\Big(\prod_{j\in G_j}\frac{p_\theta}{p_0}(X_j)\Big)^m d\Pi(\theta)\le e^{-\mathfrak{c}_1^lR^2ml\varepsilon_l^2}\Big\}.$$ 
Next, observe that
\begin{align}
\label{aux1}
&\widehat{Q}_{n,m}^{(j)}(\rho(\theta,\theta_0)\geq R\varepsilon_l)\nonumber
\\
=&\widehat{Q}_{n,m}^{(j)}(\rho(\theta,\theta_0)\geq R\varepsilon_l)\mathbf{1}_{\mathbb{B}_l}+\widehat{Q}_{n,m}^{(j)}(\rho(\theta,\theta_0)\geq R\varepsilon_l)\mathbf{1}_{\mathbb{B}_l^c}\nonumber
\\
\le&\widehat{Q}_{n,m}^{(j)}(\rho(\theta,\theta_0)\geq R\varepsilon_l)\mathbf{1}_{\mathbb{B}_l}+\mathbf{1}_{\mathbb{B}_l^c},
\end{align}
where the fact that $\widehat{Q}_{n,m}^{(j)}(\rho(\theta,\theta_0)\geq R\varepsilon_l)\le 1$ has been used.
Similarly, we have that 
\begin{align}
\label{aux2}
&\widehat{Q}_{n,m}^{(j)}(\rho(\theta,\theta_0)\geq R\varepsilon_l)\mathbf{1}_{B_l}\nonumber
\\
=& \widehat{Q}_{n,m}^{(j)}(\rho(\theta,\theta_0)\geq R\varepsilon_l)\mathbf{1}_{\mathbb{B}_l\cap \mathbb{A}_l}+\widehat{Q}_{n,m}^{(j)}(\rho(\theta,\theta_0)\geq R\varepsilon_l)\mathbf{1}_{\mathbb{B}_l\cap \mathbb{A}_l^c}\nonumber
\\
\le&\widehat{Q}_{n,m}^{(j)}(\rho(\theta,\theta_0)\geq R\varepsilon_l)\mathbf{1}_{\mathbb{B}_l\cap \mathbb{A}_l}+\mathbf{1}_{\mathbb{A}_l^c}.
\end{align}
Therefore, from Equation (\ref{aux1}) and Equation (\ref{aux2}),
\begin{align}
\label{aux3}
\widehat{Q}_{n,m}^{(l)}(\rho(\theta,\theta_0)\geq R\varepsilon_l)\le& \widehat{Q}_{n,m}^{(l)}(\rho(\theta,\theta_0)\geq R\varepsilon_l)\mathbf{1}_{\mathbb{B}_l\cap \mathbb{A}_l}+\mathbf{1}_{\mathbb{A}_l^c}+\mathbf{1}_{\mathbb{B}_l^c}
\end{align}
Now, we observe that  using  Lemma \ref{lemma1} with $v_l=ml\varepsilon_l^2$, we get that
\begin{align}
\label{equation3}
&\widehat{Q}_{n,m}^{(j)}(\rho(\theta,\theta_0)\geq R\varepsilon_l)\mathbf{1}_{\mathbb{B}_l\cap \mathbb{A}_l}\nonumber
\\
\le& \frac{1}{v_l}\Big[\mathrm{KL}(\widehat{Q}_{n,m}^{(j)},\Pi_{n,m}^{\vert G_j\vert})\mathbf{1}_{\mathbb{B}_l\cap \mathbb{A}_l}\Big]+\exp(v_l)\Big[\Pi_{n,m}(\rho(\theta,\theta_0)\geq R\varepsilon_l|G_j)\mathbf{1}_{\mathbb{B}_l\cap \mathbb{A}_l}\Big].
\end{align}
By Equation (\ref{aux3}) and Equation (\ref{equation3}) it follows that,
\begin{align*}
   & \Pr\Big(\widehat{Q}_{n,m}^{(j)}(\rho(\theta,\theta_0)\geq R\eps_l)>\exp(-ml\varepsilon_l^2)+R\varepsilon_l\Big)
    \\
    &\le \Pr(\Pi_{n,m}(\rho(\theta,\theta_0)\geq R\varepsilon_l|G_j)\mathbf{1}_{\mathbb{B}_l\cap \mathbb{A}_l}>\exp(-ml\varepsilon_l^2-v_l)) 
    \\
&+\Pr(\mathbf{1}_{\mathbb{A}_l^c}>0)+\Pr(\mathbf{1}_{\mathbb{B}_l^c}>0)+\Pr(\frac{1}{v_l}\Big[\mathrm{KL}(\widehat{Q}_{n,m}^{(j)},\Pi_{n,m}^{\vert G_j\vert})\Big]>R\varepsilon_l)
\end{align*}
Then by Markov's inequality,
\begin{align}
\label{equation1}  
  & \Pr\Big(\widehat{Q}_{n,m}^{(j)}(\rho(\theta,\theta_0)\geq R\eps_l)>\exp(-ml\varepsilon_l^2)+R\eps_l\Big)\nonumber
    \\
    &\le \frac{\mathbb{E}_0(\Pi_{n,m}(\rho(\theta,\theta_0)\geq R\varepsilon_l|G_j)\mathbf{1}_{\mathbb{B}_l\cap \mathbb{A}_l})}{\exp(-ml\varepsilon_l^2-v_l)} +P(\mathbf{1}_{\mathbb{A}_l^c}>0)+P(\mathbf{1}_{\mathbb{B}_l^c}>0)+\frac{\mathbb{E}_0\Big[\mathrm{KL}(\widehat{Q}_{n,m}^{(j)},\Pi_{n,m}^{\vert G_j\vert})\Big]}{v_lR\eps_l}\nonumber
    \\
    &=I_1+I_2+I_3+I_4.
\end{align}
Now, we bound each of the terms $I_1,$ $I_2$, $I_3$ and $I_4.$ For the term $I_1$, we notice that,
on the event $\mathbb{A}_l\cap \mathbb{B}_l$ it is satisfied 
$$\int_{\rho(\theta,\theta_0)>R\varepsilon_l}\Big(\prod_{j\in G_j}\frac{p_\theta}{p_0}(X_j)\Big)^m d\Pi(\theta)\le \exp(-\mathfrak{c}_1^lR^2ml\varepsilon_l^2)$$
and,
\begin{align*}
    \int\Big(\prod_{j\in G_j}\frac{p_\theta}{p_0}(X_j)\Big)^m d\Pi(\theta)
    \ge& \exp(-ml\varepsilon_l^2-\mathrm{KL}(Q_{n,m}^{(j)^*},\Pi)).
\end{align*}
In consequences,
\begin{align}
\label{equation2}
    \Pi_{n,m}(\rho(\theta,\theta_0)\geq R\varepsilon_l|G_j)\mathbf{1}_{\mathbb{B}_l\cap \mathbb{A}_l}\le \exp( ml\varepsilon_l^2+\mathrm{KL}(Q_{n,m}^{(j)^*},\Pi)-\mathfrak{c}_1^lR^2ml\varepsilon_l^2)
\end{align}
Therefore, by Equation (\ref{equation2}), and by Assumption \ref{Prior-cond}, we have that
\begin{align}
    \label{equation4}
    I_1
\le&\exp(3ml\varepsilon_l^2+\mathfrak{c}_5^l l\varepsilon_{l}^2-\mathfrak{c}_1^lR^2ml\varepsilon_l^2).
\end{align}
In the following line of arguments we proceed to bound $I_2$ and $I_3.$ For $I_2$, Since if Theorem \ref{th:wong} holds with $\zeta:=\varepsilon_l$, then it also holds with $\zeta:=L\varepsilon_l$ for any $L\ge1,$ we have that by Assumption \ref{test-cond}
\begin{align*}
    P_0^{(l)}(\mathbb{A}_l)=&
    P_0^{(l)}\Big(\int_{\rho(\theta,\theta_0)>R\varepsilon}\Big(\prod_{j\in G_j}\frac{p_\theta}{p_0}(X_j)\Big)^m d\Pi(\theta)\le \exp(-\mathfrak{c}_1^lR^2ml\varepsilon_l^2)\Big)
    \\
    \ge& 
    1-4\exp(-(\mathfrak{c}_2^l)^2R^2l\varepsilon_l^2).
\end{align*}
This is,
\begin{align}
    \label{equation5}
    I_2\le4\exp(-(\mathfrak{c}_2^l)^2R^2l\varepsilon_l^2).
\end{align}
Finally, we bound the term $I_3.$ For this purpose, we apply Lemma \ref{lemma1} with \break $\mathrm{F}=\log \Big(\Big(\prod_{j\in G_j}\frac{p_\theta}{p_0}(X_j)\Big)^m\Big),$ $ Q_0=Q_{n,m}^{(j)^*}$ and $\Pi_0=\Pi$ to obtain
$$
\log \int \Big(\prod_{j\in G_j}\frac{p_\theta}{p_0}(X_j)\Big)^m \mathrm{d} \Pi(\boldsymbol{\theta}) \geq \int \log \left(\Big(\prod_{j\in G_j}\frac{p_\theta}{p_0}(X_j)\Big)^m\right) \mathrm{d} Q_{n,m}^{(j)^*}(\mathrm{~d} \boldsymbol{\theta})-\mathrm{KL}(Q_{n,m}^{(j)^*}, \Pi) .
$$
Hence, we have
$$
\begin{aligned}
\mathrm{P}_{0}^{(l)}\left(\mathbb{B}_l^C\right) & =\mathrm{P}_{0}^{(l)}\left(\log \int \Big(\prod_{j\in G_j}\frac{p_\theta}{p_0}(X_j)\Big)^m \mathrm{d} \Pi(\boldsymbol{\theta})<-ml\varepsilon_l^2-\mathrm{KL}(Q_{n,m}^{(j)^*}, \Pi)\right) \\
& \leq \mathrm{P}_{0}^{(l)}\left(-\int \log \left(\Big(\prod_{j\in G_j}\frac{p_0}{p_\theta}(X_j)\Big)^m\right) \mathrm{d} Q_{n,m}^{(j)^*}(\mathrm{~d} \boldsymbol{\theta}) \leq-ml\varepsilon_l^2\right) \\
& \leq \mathrm{P}_{0}^{(l)}\left(\int 0 \vee \log \Big(\prod_{j\in G_j}\frac{p_0}{p_\theta}(X_j)\Big) \mathrm{d} Q_{n,m}^{(j)^*}(\boldsymbol{\theta}) \geq l\varepsilon_l^2\right) \\
& \leq \frac{1}{l\varepsilon_l^2} \mathrm{P}_{0}^{(l)}\left[\int 0 \vee \log \Big(\prod_{j\in G_j}\frac{p_0}{p_\theta}(X_j)\Big) \mathrm{d} Q_{n,m}^{(j)^*}(\boldsymbol{\theta})\right] \\
& \leq \frac{1}{l\varepsilon_l^2} Q_{n,m}^{(j)^*}\left[\mathrm{KL}\left(\mathrm{P}_{0}^{(l)}, \mathrm{P}_{\theta}^{(l)}\right)+\sqrt{\frac{1}{2}\mathrm{KL}\left(\mathrm{P}_{0}^{(l)}, \mathrm{P}_{\theta}^{(l)}\right)}\right] \\
& \leq \frac{1}{l\varepsilon_l^2}\left(2 Q_{n,m}^{(j)^*}\left[\mathrm{KL}\left(\mathrm{P}_{0}^{(l)}, \mathrm{P}_{\theta}^{(l)}\right)\right]+1\right) .
\end{aligned}
$$
Here we use Markov's inequality for the fourth line and use Fubini's theorem and Lemma B.13 of \cite{ghosal2017fundamentals} for the fifth line. For the last line, we use a simple inequality $z+\sqrt{z / 2} \leq z+(z \vee 1) \leq 2 z+1$ for $z \geq 0$. 
Thus, by Assumption \ref{Prior-cond}
\begin{align}
    \label{equation6}
    I_3=\mathrm{P}_{0}^{(l)}\left(\mathbb{B}_l^C\right)\le \frac{1}{l\varepsilon_l^2}+2\mathfrak{c}_5^{l}\frac{1}{m}
\end{align}
Finally, we analyze the term $I_4.$ Using Proposition \ref{prop1}, we get that
\begin{align}
\label{equation8}I_4=&\frac{\mathbb{E}_0\Big[\mathrm{KL}(\widehat{Q}_{n,m}^{(j)},\Pi_{n,m}^{\vert G_j\vert})\Big]}{v_lR\varepsilon_l}\le \frac{\mathfrak{c}_5^{l}l\varepsilon_l^2}{ml\varepsilon_l^2R\varepsilon_l}=\frac{\mathfrak{c}_5^{l}}{Rm\eps_l}.
\end{align}
To conclude, we observe that by Equation (\ref{equation1}), (\ref{equation4}), (\ref{equation5}), (\ref{equation6}), and (\ref{equation8}), we obtain that
\begin{align*}
    & \Pr\Big(\widehat{Q}_{n,m}^{(j)}(\rho(\theta,\theta_0)\geq R\eps_l)>\exp(-ml\varepsilon_l^2)+R\varepsilon_l\Big)\nonumber
    \\
    &\le\exp(3ml\varepsilon_l^2+\mathfrak{c}_5^l l\varepsilon_{l}^2-\mathfrak{c}_1^lR^2ml\varepsilon_l^2)
    \\
    &+4\exp(-(\mathfrak{c}_2^l)^2R^2l\varepsilon_l^2)+\frac{1}{l\varepsilon_l^2}+2\mathfrak{c}_5^{l}\frac{1}{m}+\frac{\mathfrak{c}_5^{l}}{Rm\eps_l}.
\end{align*}
Furthermore, using the fact that $l\varepsilon_l^2\ge 1$ and considering $R$ such that 
$R>\sqrt{\frac{6+2\mathfrak{c}_5^l}{\mathfrak{c}_1^l}},$
we get that
\begin{align}
    \label{equation9}
    & \Pr\Big(\widehat{Q}_{n,m}^{(j)}(\rho(\theta,\theta_0)\geq A_l\eps_l)>\exp(-ml\varepsilon_l^2)+l\varepsilon_l^2\Big)\nonumber
    \\
    &\le\exp(-\frac{\mathfrak{c}_1^l}{2}R^2l\varepsilon_l^2)+4\exp(-(\mathfrak{c}_2^l)^2R^2l\varepsilon_l^2)+\frac{1}{l\varepsilon_l^2}+2\mathfrak{c}_5^{l}\frac{1}{m}+\frac{\mathfrak{c}_5^{l}}{Rm\eps_l},
\end{align}
and the claim is followed.
\end{proof}
\newpage



\section{Proof of Theorem \ref{BVM-theorem}}
\begin{proof}
We now proceed to prove part \( a) \). To achieve this, we first analyze the Wasserstein metric median \( Q^*_{\text{Met,GG}} \).
By Theorem \ref{theorem1BVM} and Assumption \ref{assumption1BVM}, convergence in total variation distance implies convergence of expectations in \( P_{\theta_0} \)-probability. Specifically, we have,
\[
\left\|\int_{\Theta} \theta \left( d Q_{n,m}^{(j),\text{GG}}(\theta) - d N\left(\theta_0 + \frac{\Delta_{l, \theta_0}}{\sqrt{l}}, \frac{1}{l \cdot m} I^{-1}\left(\theta_0\right)\right)(\theta) \right) \right\|_2 \rightarrow 0 \quad \text{as } l \rightarrow \infty.
\]
Next, note that the total variation distance between two Gaussian distributions \( N(\mu_1, \Sigma) \) and \( N(\mu_2, \Sigma) \), sharing the same covariance matrix, is bounded by a multiple of \( \|\mu_1 - \mu_2\|_2 \). Therefore, we can replace \( \theta_0 + \frac{\Delta_{l, \theta_0}}{\sqrt{l}} \) in the above result with the mean,
\[
\bar{\theta}_{j, m}^{\text{GG}}(G_j) := \int_{\Theta} \theta \, d Q_{n,m}^{(j),\text{GG}}(\theta).
\]
This substitution allows us to reformulate the conclusion of Theorem \ref{theorem1BVM} as,
\[
\left\|Q_{n,m}^{(j),\text{GG}}(\cdot) - N\left(\bar{\theta}_{j, m}^{\text{GG}}(G_j), \frac{1}{l \cdot m} I^{-1}(\theta_0)\right)\right\|_{\mathrm{TV}} \rightarrow 0 \quad \text{as } l \rightarrow \infty,
\]
in \( P_{\theta_0} \)-probability.
Now, consider \( m = \lfloor n / l \rfloor \) as fixed, and let \( n, l \rightarrow \infty \). As before, let \( G_1, \ldots, G_m \) denote disjoint groups of i.i.d. observations from \( P_{\theta_0} \), each of cardinality \( l \). Recall that by the definition of \( Q^*_{\text{Met,GG}} \) in Equation (\ref{MetricMedian}), \( Q^*_{\text{Met,GG}} = Q_{n,m}^{(j_*),\text{GG}} \) for some \( j_* \leq m \), and its mean is given by \( \theta^*_{\text{GG}} := \bar{\theta}_{j_*, m}^{\text{GG}}(G_{j_*}) \).
By definition, we have:
\begin{align}
   & \left\|Q^*_{\text{Met,GG}} - N\left(\theta^*_{\text{GG}}, \frac{1}{l \cdot m} I^{-1}(\theta_0)\right)\right\|_{\mathrm{TV}} \nonumber
    \\
    &\leq \max_{j=1, \ldots, m} \left\|Q_{n,m}^{(j),\text{GG}}(\cdot) - N\left(\bar{\theta}_{j, m}^{\text{GG}}(G_j), \frac{1}{l \cdot m} I^{-1}(\theta_0)\right)\right\|_{\mathrm{TV}}.
\end{align}
Since the right-hand side converges to zero as \( l \to \infty \), we conclude:
\[
\left\|Q^*_{\text{Met,GG}} - N\left(\theta^*_{\text{GG}}, \frac{1}{l \cdot m} I^{-1}(\theta_0)\right)\right\|_{\mathrm{TV}} \rightarrow 0 \quad \text{as } n \to \infty.
\] 
This completes the proof for \( Q^*_{\text{Met,GG}} \). The proof for \( Q^*_{\text{Met,MF}} \) follows analogously by substituting \( Q_{n,m}^{(j),\text{GG}} \) with \( Q_{n,m}^{(j),\text{MF}} \), and replacing \( I^{-1}(\theta_0) \) with \( I^{\prime-1}(\theta_0) \), where \( I^{\prime-1}(\theta_0) \) is a diagonal matrix matching the diagonal elements of \( I^{-1}(\theta_0) \). All other steps remain identical, ensuring the same conclusion,
\[
\left\|Q^*_{\text{Met,MF}} - N\left(\theta^*_{\text{MF}}, \frac{1}{l \cdot m} I^{\prime-1}(\theta_0)\right)\right\|_{\mathrm{TV}} \rightarrow 0 \quad \text{as } n \to \infty.
\]
This completes the proof of part \( a) \).

For part \( b)\), we proceed to analyze the mean of the Wasserstein metric median \( Q^*_{\text{Met,GG}}\) under the assumptions of the theorem. By assumption, we have that \( \varepsilon_l \) satisfies \( l \varepsilon_l^2 \geq 1 \) and \( \sqrt{n} \leq m \). These conditions imply that, for any \( 1 \leq j \leq \lfloor(1-\kappa)m\rfloor+1 \), the following inequality holds,
\[
\exp\left(-\frac{\mathfrak{c}_1^l}{2} R^2 l \varepsilon_l^2\right) + 4 \exp\left(-(\mathfrak{c}_2^l)^2 R^2 l \varepsilon_l^2\right) + \frac{1}{l \varepsilon_l^2} + 2 \mathfrak{c}_5^l \frac{1}{m} + \frac{\mathfrak{c}_5^l}{R m \varepsilon_l} \leq \frac{1}{4}.
\]
Applying Theorem \ref{mainT}, this leads to the bound,
\[
\Pr\left(d_{W_{1,\rho}}(\widehat{Q}_{n,m}^{(j)\text{GG}}, \delta_0) \geq 2R\varepsilon_l + \exp(-ml\varepsilon_l^2)\right) \leq \frac{1}{4}.
\]
Using Theorem \ref{thm-rob-mm}, we then conclude:
\[
\Pr\left(d_{W_{1,\rho}}(Q^*_{\text{Met,GG}}, \delta_0) > 3\left(2R\varepsilon_l + \exp(-ml\varepsilon_l^2)\right)\right) \leq \left[ e^{(1 - \kappa) \psi\left( \frac{1/2 - \kappa}{1 - \kappa}, \frac{1}{4} \right)} \right]^{-m}.
\]
This implies that the Wasserstein distance between \( Q^*_{\text{Met,GG}} \) and \( \delta_0 \) satisfies,
\[
d_{W_{1,\rho}}(Q^*_{\text{Met,GG}}, \delta_0) \leq 3\left(2R\varepsilon_l + \exp(-ml\varepsilon_l^2)\right),
\]
with probability at least \( 1 - \left[ e^{(1 - \kappa) \psi\left( \frac{1/2 - \kappa}{1 - \kappa}, \frac{1}{4} \right)} \right]^{-m}.
\)
Finally, noting the relationship between the Wasserstein distance and the Euclidean norm, we have
\[
\|\theta_{\text{GG}}^* - \theta_0\|_2 \leq d_{W_{1,\rho}}(Q^*_{\text{Met,GG}}, \delta_0).
\]
This establishes the finite-sample confidence bound for \( \|\theta_{\text{GG}}^* - \theta_0\|_2 \), concluding the proof for \( \theta_{\text{GG}}^* \).
The proof for \( \theta_{\text{MF}}^* \) is entirely analogous.
This completes the proof of part \( b) \) and we conclude the result.
\end{proof}

\newpage
\section{Auxiliary Lemmas}
\begin{lemma}
    \label{lemma1}
     Let $\Theta$ be a measurable space. Then for any two distributions $Q_0, \Pi_0 \in \mathcal{P}(\Theta)$ and any measurable function $\mathrm{F}: \Theta \mapsto \mathbb{R}$,
$$
Q_0[\mathrm{~F}] \leq \mathrm{KL}\left(Q_0, \Pi_0\right)+\log \left(\Pi_0\left[\mathrm{e}^{\mathrm{F}}\right]\right) .
$$
In particular, for any measurable subset $\Theta^{\prime} \subset \Theta$ and positive constant $v>0$,
$$
Q_0\left(\Theta^{\prime}\right) \leq \frac{1}{v}\left\{\mathrm{KL}\left(Q_0, \Pi_0\right)+\mathrm{e}^v \Pi_0\left(\Theta^{\prime}\right)\right\}
$$
\end{lemma}
\begin{proof}
   Consider the case $Q_0$ is not absolutely continuous with respect to $\Pi_0$ then $\operatorname{KL}\left(Q_0, \Pi_0\right)=\infty$, so the result trivially holds. Now assume otherwise. Otherwise, recall the following well-known duality formula (see Lemma 2.2 of \cite{10.1214/19-AOS1855}),
$$
\log \left(\Pi_0\left[\mathrm{e}^{\mathrm{F}}\right]\right)=\sup _{Q^{\prime} \ll \Pi_0}\left[Q^{\prime}[\mathrm{F}]-\mathrm{KL}\left(Q^{\prime}, \Pi_0\right)\right],
$$
from where the first assertion is directly followed. For the proof of the last part of the lemma, we let $F(\boldsymbol{\theta}):=v \mathbf{1}\left(\boldsymbol{\theta} \in \Theta^{\prime}\right)$. Then we have
$$
\mathrm{e}^v \Pi_0\left(\Theta^{\prime}\right) \geq \log \left(1+\mathrm{e}^v \Pi_0\left(\Theta^{\prime}\right)\right) \geq \log \left(\int \mathrm{e}^{\mathrm{F}(\boldsymbol{\theta})} \mathrm{d} \Pi_0(\boldsymbol{\theta})\right),
$$
which completes the proof.
\end{proof}
\newpage
\section{Proof of Proposition \ref{prop1}}
\begin{proof}
 For any distribution $Q \in \mathcal{Q}$, we have the following series of inequalities

\begin{align*}
\mathrm{P}_{0}^{(l)} & {\left[\mathrm{KL}\left(Q, \Pi_{n,m}\left(\cdot \mid G_n\right)\right)\right] } \\
& =\int \mathrm{P}_{0}^{(l)}\left[\log \left(\frac{\int\limits_{\Theta}\left(\prod_{i\in G_j} p_\theta(X_i)\right)^m d\Pi(\theta) \mathrm{d} Q(\boldsymbol{\theta})}{
\left(\prod_{i\in G_j} p_\theta(X_i)\right)^m
\mathrm{d} \Pi(\boldsymbol{\theta})}\right)\right] \mathrm{d} Q(\boldsymbol{\theta}) \\
& =\mathrm{KL}\left(Q, \Pi\right)+\int \mathrm{P}_{0}^{(l)}\left[\log \left(\frac{\left(\prod_{i\in G_j} p_0(X_i)\right)^m}{\left(\prod_{i\in G_j} p_\theta(X_i)\right)^m}\right)\right] \mathrm{d} Q(\boldsymbol{\theta})
\\
&
+\mathrm{P}_{0}^{(l)}\left[\log \left(\frac{\int\limits_{\Theta}\left(\prod_{i\in G_j} p_\theta(X_i)\right)^m d\Pi(\theta)}{\left(\prod_{i\in G_j} p_0(X_i)\right)^m}\right)\right]
\\
& =\mathrm{KL}\left(Q, \Pi\right)+m\int \mathrm{P}_{0}^{(l)}\left[\log \left(\frac{\prod_{i\in G_j} p_0(X_i)}{\prod_{i\in G_j} p_\theta(X_i)}\right)\right] \mathrm{d} Q(\boldsymbol{\theta})
\\
&+\mathrm{P}_{0}^{(l)}\left[\log \left(\frac{\int\limits_{\Theta}\left(\prod_{i\in G_j} p_\theta(X_i)\right)^m d\Pi(\theta)}{\left(\prod_{i\in G_j} p_0(X_i)\right)^m}\right)\right]
\\
& =\mathrm{KL}\left(Q, \Pi\right)+mQ\left[\mathrm{KL}\left(\mathrm{P}_{0}^{(l)}, \mathrm{P}_{\theta}^{(l)}\right)\right]+\mathrm{P}_{0}^{(l)}\left[\log \left(\frac{\int\limits_{\Theta}\left(\prod_{i\in G_j} p_\theta(X_i)\right)^m d\Pi(\theta)}{\left(\prod_{i\in G_j} p_0(X_i)\right)^m}\right)\right].
\end{align*}
Then, we observe that by Jensen Inequality 
$$\mathrm{P}_{0}^{(l)}\left[\log \left(\frac{\int\limits_{\Theta}\left(\prod_{i\in G_j} p_\theta(X_i)\right)^m d\Pi(\theta)}{\left(\prod_{i\in G_j} p_0(X_i)\right)^m}\right)\right]\le\log\left[\mathrm{P}_{0}^{(l)} \left(\frac{\int\limits_{\Theta}\left(\prod_{i\in G_j} p_\theta(X_i)\right)^m d\Pi(\theta)}{\left(\prod_{i\in G_j} p_0(X_i)\right)^m}\right)\right]\le 0,$$
and in consequence,
$$\mathrm{P}_{0}^{(l)} {\left[\mathrm{KL}\left(Q, \Pi_{n,m}\left(\cdot \mid G_n\right)\right)\right] } \le \mathrm{KL}\left(Q, \Pi\right)+mQ\left[\mathrm{KL}\left(\mathrm{P}_{0}^{(l)}, \mathrm{P}_{\theta}^{(l)}\right)\right].$$
Thus, by the definition of $\widehat{Q}_n$,
$$
\begin{aligned}
\mathrm{P}_{0}^{(l)}\left[\mathrm{KL}\left(\widehat{Q}_{n,m}^{(j)}, \Pi_{n,m}\left(\cdot \mid G_j\right)\right)\right] & =\mathrm{P}_{0}^{(l)}\left[\inf _{Q \in \mathcal{Q}} \mathrm{KL}\left(Q, \Pi_{n,m}\left(\cdot \mid G_j\right)\right)\right] \\
& \leq \inf _{Q \in \mathcal{Q}} \mathrm{P}_{0}^{(l)}\left[\mathrm{KL}\left(Q, \Pi_{n,m}\left(\cdot \mid G_j\right)\right)\right] \\
& \leq \inf _{Q \in \mathcal{Q}}\left\{\mathrm{KL}\left(Q, \Pi\right)+mQ\left[\mathrm{KL}\left(\mathrm{P}_{0}^{(l)}, \mathrm{P}_{\theta}^{(l)}\right)\right]\right\},
\end{aligned}
$$
which proves the first desired result. The second assertion. follows from
$$
\inf _{Q \in \mathcal{Q}}\left\{\mathrm{KL}\left(Q, \Pi\right)+mQ\left[\mathrm{KL}\left(\mathrm{P}_{0}^{(l)}, \mathrm{P}_{\theta}^{(l)}\right)\right]\right\}\le \mathrm{KL}\left(Q_{n, m}^{(j)^*}, \Pi\right)+mQ_{n, m}^{(j)^*}\left[\mathrm{KL}\left(\mathrm{P}_{0}^{(l)}, \mathrm{P}_{\theta}^{(l)}\right)\right],
$$
and Assumption \ref{Prior-cond}.
\end{proof}

\vskip 0.2in
\bibliography{bibliography1,bibliography2,papers}

\end{document}